\documentclass[a4paper,11pt,fleqn]{article}
\usepackage[utf8]{inputenc} 
\usepackage{url}            
\usepackage{amsfonts}       
\usepackage{euscript}
\usepackage{charter}
\usepackage{nicefrac}       
\usepackage{microtype}      
\usepackage{xcolor}         

\usepackage{amsmath,amssymb,amsthm}
\usepackage{thmtools,thm-restate}
\usepackage{enumitem}
\usepackage{graphicx}
\usepackage[ruled,vlined,linesnumbered]{algorithm2e}

\definecolor{labelkey}{rgb}{0,0.08,0.45}
\definecolor{refkey}{rgb}{0,0.6,0.0}
\definecolor{Brown}{rgb}{0.45,0.0,0.05}
\definecolor{dgreen}{rgb}{0.00,0.49,0.00}
\definecolor{dblue}{rgb}{0,0.08,0.75}
\RequirePackage[colorlinks,hyperindex]{hyperref} 
\hypersetup{linktocpage=true,citecolor=dblue,linkcolor=dgreen}

\numberwithin{equation}{section}

\DeclareMathOperator*{\argmin}{\ensuremath{\mathrm{arg\,min}}}
\DeclareMathOperator*{\argmax}{\ensuremath{\mathrm{arg\,max}}}

\newcommand{\dom}{\ensuremath{\text{\rm dom}\,}}
\newcommand{\range}{\ensuremath{\text{\rm Im}}}
\newcommand{\inte}{\ensuremath{\text{\rm int}}}

\newcommand{\Ker}{\ensuremath{\text{\rm Ker}}}

\newcommand{\KL}{\ensuremath{\mathsf{KL}}}
\newcommand{\Ent}{\ensuremath{\mathsf{H}}}
\newcommand{\sgn}{\ensuremath{\text{\rm sgn}}}

\newcommand{\Pc}{{\cal P}}
\newcommand{\Xx}{{\mathbb X}}

\newcommand{\X}{X}
\newcommand{\Y}{Y}

\newcommand{\Cc}{{\cal C}}
\newcommand{\Ic}{{\cal I}}

\newcommand{\R}{\mathbb R}
\newcommand{\extR}{\left]-\infty,+\infty\right]}

\newcommand{\N}{\mathbb N}

\newcommand{\M}{\mathsf R}

\newcommand{\Xs}{\pi}
\newcommand{\vs}{\mathsf{v}}
\newcommand{\us}{\mathsf{u}}

\newcommand{\ps}{\mathsf{p}}
\newcommand{\rs}{\mathsf{r}}
\newcommand{\As}{\mathsf{A}}

\newcommand{\Ss}{\mathsf{S}}
\newcommand{\Ks}{\xi}
\newcommand{\reg}{\eta}
\newcommand{\Cs}{\mathsf{C}}
\newcommand{\Vs}{\mathsf{V}}
\newcommand{\as}{\mathsf{a}}
\newcommand{\xs}{\pi}
\newcommand{\ys}{\gamma}

\newcommand{\bigO}{{\cal O}}

\newcommand{\Leg}{\phi}

\providecommand{\norm}[1]{\lVert#1\rVert}
\providecommand{\scalarp}[1]{\langle#1\rangle}
\providecommand{\abs}[1]{\lvert#1\rvert}

\newtheorem{theorem}{Theorem}[section]

\newtheorem{lemma}[theorem]{Lemma}
\newtheorem{fact}[theorem]{Fact}
\newtheorem{proposition}[theorem]{Proposition}

\newtheorem{remark}[theorem]{Remark}

\usepackage[textwidth=2.0cm, textsize=tiny]{todonotes} 

\title{ { \sffamily Convergence of Batch Greenkhorn \\ for Regularized Multimarginal Optimal Transport} } 

\author{Vladimir Kostic\thanks{Istituto Italiano di Tecnologia, Via Melen, 83,  
		16152 Genova, Italy 
		({\tt vladimir.kostic@iit.it}) and Department of Mathematics and Informatics, Faculty of Science, University of Novi Sad, Trg Dositeja Obradovića 4, 21000 Novi Sad, Serbia.} \and Saverio Salzo\thanks{Istituto Italiano di Tecnologia, Via Melen, 83,  
        16152 Genova, Italy ({\tt saverio.salzo@iit.it}) and Department of Computer Science, University College London.}
        \and Massimilano Pontil\thanks{Istituto Italiano di Tecnologia, Via Melen, 83,  
        16152 Genova, Italy ({\tt massimilano.pontil@iit.it}) and Department of Computer Science, University College London.}}
\date{}

\begin{document}
\maketitle

\begin{abstract}
In this work we propose a batch version of the Greenkhorn
algorithm for multimarginal regularized optimal transport problems.
Our framework is general enough
to cover, as particular cases, some existing algorithms 
like Sinkhorn and Greenkhorn algorithm for the bi-marginal setting, and (greedy) MultiSinkhorn for multimarginal optimal transport.
We provide a complete convergence analysis,
which is based on the properties of
the iterative Bregman projections (IBP) method with greedy control.
Global linear rate of convergence 
and explicit bound on the iteration complexity are obtained.
When specialized to above mentioned algorithms,
our results give new insights and/or improve
existing ones.
\end{abstract}

 \vspace{1ex}

\section{Introduction}
Over the recent years the field of optimal transport  (OT)~\cite{Vil2009} has received significant attention in machine learning and data science, as it provides natural and powerful tools to compare probability distributions. In this paper we study a general class of OT problems known as multimarginal optimal transport (MOT), whereby several probability distributions are coupled together in order to compute a measure of their association, see e.g. \cite{Pas2015, BCCN2017}. 
MOT is receiving increasing interest due to its numerous applications, ranging from density functional theory in quantum chemistry, to fluid dynamics, to economics, to image processing, just to name a few, see \cite{PC2019} and 
references therein. Particularly, in machine learning, MOT is important for generative adversarial networks (GANs) \cite{CMZJST2019}, domain adaptation \cite{HZKSC2019}, Wasserstein barycenters \cite{AC2011}, 
clustering \cite{MB2021}, and  Bayesian inference of joint distributions \cite{FP2019},
among others.

We focus on discrete MOT, in which, given $m$ finitely supported probability distributions, we wish to compute an optimal joint distribution which solves a linear program, whose objective function  
involves a cost tensor and the constraint set requires the joint distribution to have the given individual ones as its marginals. It is well known that addressing directly the MOT problem is computationally prohibitive. Furthermore, unlike the bi-marginal case, MOT is NP-Hard for certain costs, even approximately \cite{AB2020}. To overcome this issue, regularization techniques have been widely considered. The key insight is to add a strongly convex regularizer to the MOT objective. A popular choice in the bi-marginal setting is entropic regularization, which leads to the well-known Sinkhorn divergence.  
However, the study of computational regularized multimarginal optimal transport (RMOT) is less developed, and the objective of this work is to devise efficient algorithms with convergence guarantees.

\paragraph{Related work.} Two popular frameworks for solving regularized OT problems are iterative Bregman projections (IBP) and alternating dual minimization. Both have been extensively studied in the bi-marginal case and two popular offsprings include Sinkhorn and Greenkhorn algorithms. In contrast, the theory of general multimarginal OT is less understood. The multimarginal version of Sinkhron algorithm was first proposed by \cite{BCCNP2015} and its convergence was established within the framework of \emph{cyclic IBP} in finite dimension. More recently, the work \cite{DG2020} extended the result to infinite dimension and in \cite{Car2021} even global linear rate of convergence was obtained. On the other hand, the work \cite{LHCJ2020} proposed a version of  \emph{multimarginal Sinkhorn} where, at each iteration, the marginal to project on is chosen in a \textit{greedy} way. In the same work acceleration using Nesterov's estimate sequences was also proposed and near-linear time computational complexity bounds were obtained. Finally, alternating dual minimization framework was used in \cite{TDGU2020} with Nesterov's momentum acceleration to build a competitive algorithm against the greedy multimarginal Sinkhorn. Both papers \cite{LHCJ2020,TDGU2020} derived sub-linear computational complexity bounds for the considered RMOT algorithms, but no linear convergence rates were studied.

\paragraph{Contribution.} After the introduction of RMOT as a Bregman projection problem in Sec. \ref{sec:ermot}, in Sec. \ref{sec:bg} we propose a batch version of the popular Greenkhorn algorithm for RMOT that greedily selects at each iteration a marginal and a batch of its components of prescribed sizes. In Sec. \ref{sec:conv}, we establish its linear convergence, providing also explicit rates depending on problem data in two important cases: 1) batch Greenkhorn for bi-marginal optimal transport and 2) (greedy) MultiSinkhorn of \cite{LHCJ2020}. Moreover, in these two cases we also provide iteration complexity bounds that improve state-of-the-art results.
We stress that the explicit rate we obtain for the greedy multimarginal Sinkhorn is strictly better than the one recently derived in \cite{Car2021}, for cyclic multimarginal Sinkhorn. This shows the effectiveness of the greedy option for speeding up the convergence of this type of algorithms.

\paragraph{Notation.} Given $n\in\N$, let $[n]:=\{1,2,\ldots,n\}$, and let $\Delta_n = \{x\in\R_+^n \,\vert\, \norm{x}_1 = 1\}$ be the unit simplex. Next, let $m\geq2$ and $n_1, \dots, n_m\in\N$. The space of tensors (of order $m$) is denoted by $\Xx := \R^{n_1\times\cdots\times n_m}$. 
For the sake of brevity we set $\mathcal{J} = [n_1]\times \cdots \times [n_m]$ and we denote by $j = (j_1, \dots, j_m)$ a general multi-index in $ \mathcal{J}$. 
We  set $\Xx_+ = \{ \xs \in \Xx \,\vert\, \xs_j \geq 0, \text{ for every } j \in \mathcal{J} \}$ and 
$\Xx_{++} = \{ \xs \in \Xx \,\vert\, \xs_j > 0, \text{ for every } j \in \mathcal{J} \}$.
We need also to consider multi-indexes without the $k$-th index, so we define $\mathcal{J}_{-k} = [n_1] \times\dots\times [n_{k-1}]\times[n_{k+1}]\times\dots\times[n_m]$ and $j_{-k}$ a general multi-index in $\mathcal{J}_{-k}$.
We will identify $\mathcal{J}$ with $\mathcal{J}_{-k} \times [n_k]$
via the mapping $j\ \leftrightarrow\ (j_{-k}, j_k)$ with $j_{-k} = (j_1,\dots, j_{k-1}, j_{k+1}, \dots, j_m)$.
The space $\Xx$ is an Euclidean space when endowed with the standard scalar product and norm, defined as
\begin{equation}
\scalarp{\xs,\xs^\prime}=
\sum_{j \in \mathcal{J}} \xs_j
\quad \norm{\xs}^2 =
\sum_{j \in \mathcal{J}}
\xs_j^2
\qquad\forall\,\xs,\xs^\prime\in\Xx. 
\end{equation}
For every $\xs,\xs^\prime$ we denote by $\xs\odot\xs^\prime \in \Xx$, the Hadamard product of $\xs$ and $\xs^\prime$,
that is, $(\xs\odot\xs^\prime)_{j} = \xs_{j} \xs^\prime_{j}$. Moreover, 
if $\mathsf{v}_1 = (v_{1,j_1})_{j_1 \in [n_1]} \in \R^{n_1}, \dots, \mathsf{v}_m = (v_{m, j_m}))_{j_m \in [n_m]} \in \R^{n_m}$,
we set $\oplus_{k=1}^m \mathsf{v}_k \in \Xx$ such that
$(\oplus_{k=1}^m \mathsf{v}_k)_{j} = \sum_{k=1}^m v_{k,j_k}$, and $\otimes_{k=1}^m \mathsf{v}_k \in \Xx$ such that
$(\otimes_{k=1}^m \mathsf{v}_k)_{j} = \prod_{k=1}^m v_{k,j_k}$
For a function $\phi\colon \Xx\to \left]-\infty,+\infty\right]$, the Fenchel conjugate is denoted by $\phi^*$. Finally, as usual, the Dirac measure at $x$
is denoted by $\delta_x$.

\section{Entropic RMOT as the Bregman projection problem}\label{sec:ermot}

In this section we formally introduce the discrete multimarginal optimal transport problem and its regularized versions, emphasizing its connection with the Bregman projection problem.

Let $\as_k \in \Delta_{n_k}$, $k\in [m]$ be prescribed histograms. MOT consists in solving a linear program 
  \begin{equation}
 \label{MOT}
 \min_{\Xs\in \Pi(\as_1, \dots, \as_m)} \scalarp{\Cs,\Xs},
 \end{equation}
where $\Cs \in \Xx$ is a given cost tensor and $\Pi(\as_1, \dots, \as_m)$, called \emph{transport polytope}, is the covex set of nonnegative tensors in $\Xx$ whose marginals are 
$\as_1,\ldots,\as_m$. More specifically
\begin{equation}
\label{eq:transpoly}
\Pi(\as_1,\dots,\as_m) 
= \big\{\Xs\in \Xx_+\,\big\vert\, \M_k(\Xs)=\as_k, \text{ for every } k \in [m]
\big\},
\end{equation}
where for all $k \in [m]$, we denote by $\M_k \colon \Xx \to \R^{n_k}$ the \emph{$k$-th push-forward projection operator}, defined as
\begin{equation}
\label{eq:pushforw}
(\forall\, j_k\in [n_k]) \quad\M_k(\pi)_{j_k} = 
\sum_{j_{-k} \in \mathcal{J}_{-k}} \xs_{(j_{-k}, j_k)},
\end{equation}
so that $\M_k(\pi)$ is the $k$-th marginal of $\pi$.

As noted in the introduction, problem \eqref{MOT} may be hard to solve. We thus consider a regularized version using the (negative) Boltzmann-Shannon entropy $\Ent(\pi):= \sum_{j} \xs_{j} (\log \xs_{j} -1)$, $\xs\in\dom \Ent = \Xx_+$.
For a cost tensor $\Cs \in \Xx_+$, the regularized multimarginal optimal transport (RMOT) problem consists in computing
 \begin{equation}
 \label{RMOT}
 \Xs^\star = 
 \argmin_{\Xs\in \Pi(\as_1,\dots,\as_m)} 
 \scalarp{\Cs,\Xs} + \reg \Ent(\Xs),
 \end{equation}
where $\reg>0$ is a regularization parameter. 
Now, recalling that the Kulback-Leibler (KL) divergence $\KL\colon\Xx\times\Xx\to[0,+\infty]$ is the \emph{Bregman distance} associated to the (negative) entropy $\Ent$, 
i.e.,
\begin{align}
\label{KLdiv}
\nonumber
\KL(\Xs,\Xs^\prime) &=
\begin{cases}
\Ent(\Xs)-\Ent(\Xs^\prime)-\scalarp{\Xs-\Xs^\prime,\nabla \Ent(\Xs^\prime)}&\text{if } \Xs^\prime \in \inte(\dom \Ent)\\
+\infty &\text{otherwise}
\end{cases}\\[1ex]
&= \begin{cases}
\sum_{j \in \mathcal{J}} \Xs_j \log (\Xs_j/\Xs^\prime_j) - \Xs_j + \Xs^\prime_j
&\text{if } \Xs^\prime \in \Xx_{++}, \Xs \in \Xx_{+}\\
+\infty &\text{otherwise},
\end{cases}
\end{align} 
and considering that $\dom \Ent = \Xx_+$,
it easy to check that  \eqref{RMOT} can be re-written as
 \begin{equation}
 \label{eq:RMOT_BP}
 \Xs^\star =\argmin_{\M_k(\Xs) = \as_k, k \in [m]} \KL(\Xs,\Ks), 
 \end{equation}
where $\Ks = \nabla\Ent^*(- \Cs / \reg) = \exp (- \Cs / \reg) \in \Xx_+$ is the \emph{Gibbs kernel} tensor. Hence, the solution of \eqref{eq:RMOT_BP} is nothing but the \emph{KL projection of $\Ks$} onto the affine set  $\{\Xs \in \Xx \,\vert\, (\forall\, k \in [m])\  \M_k(\Xs) = \as_k\}$. 
Indeed, in general, for an arbitrary closed convex set $\Cc\subset \Xx$ 
such that $\Cs\cap \Xx_{++} \neq \varnothing$ and a point $\pi\in\Xx_{++}$, the KL projection of $\pi$ onto $\Cc$ is defined as
\begin{equation}
\label{eq:BregProj}
\Pc_\Cc(\xs) := \argmin_{\gamma\in\Cc}\KL(\gamma,\pi),
\end{equation}
while its KL distance to $\Cc$ is defined as $\KL_{\Cc}(\Xs):= \KL(\Pc_{\Cc}(\Xs),\Xs)$.

\section{Greedy KL projections for entropic RMOT}\label{sec:bg}
Now we focus on the entropic-regularized MOT problem \eqref{RMOT},
in the equivalent form \eqref{eq:RMOT_BP}. 
Recalling definition \eqref{eq:pushforw}, we
introduce the linear operator
\begin{equation}
\label{eq:Roper}
\M\colon\Xx \to \R^{n_1}\times\cdots\times\R^{n_m},
\qquad \M(\xs)=(\M_1(\xs),\dots,\M_m(\xs)).
\end{equation}
Then, the constraints in \eqref{eq:RMOT_BP} define the affine set
\begin{equation}
\label{eq:20211128l}
\Pi = \{ \Xs\in \Xx \;\vert\; \M(\xs) = (\as_1,\dots, \as_m) \}.
\end{equation}
We recall that the positiveness constrain embodied in problem \eqref{RMOT} can be absorbed in the entropy function $\Ent$, since $\dom \Ent = \Xx_+$, so that problem \eqref{RMOT} is equivalent to the computation of the KL projection of the Gibbs kernel $\xi$ onto the affine set $\Pi$. However, it is possible to prove (see equation \eqref{eq:same_proj} in Appendix~\ref{app:bregman}) that
$\Pc_{\Pi}(\xi) = \Pc_{\Pi}(\xi \odot \otimes_{k=1}^m \as_k)$.
Therefore, we will equivalently target the computation of the KL projection of 
the normalized Gibbs kernel
$\xi \odot \otimes_{k=1}^m \as_k$.
In order to compute such projection, in the following, we will represent
the set $\Pi$ as an intersection of simpler affine sets
and then we will use the alternating KL projection algorithm. More precisely, we will rewrite the set \eqref{eq:20211128l}  as an intersection of affine sets, obtained via specific \emph{sketching}, on which KL projections have closed forms. 

We consider sketches that cover several popular algorithms (multimarginal Sinkhorn of \cite{BCCNP2015}, greedy MultiSinkhorn of \cite{LHCJ2020}, bi-marginal Greenkhorn of \cite{AWR2017}). 
To that purpose, for each $k \in [m]$ (which refers to the $k$-th marginal) and each batch $L \subset [n_k]$
we consider the \emph{canonical injection}
\begin{equation*}
\Ss_{(k,L)}\colon\R^L \to \R^{n_1}\times \cdots\times\R^{n_m} 
\end{equation*}
 of $\R^L$  into $\R^{n_1}\times \cdots\times\R^{n_m}$, meaning that for each 
 $\us=(u_{j_k})_{j_k \in L} \in \R^L$,
$\Ss_{(k,L)}\us$ is the vector of $\R^{n_1}\times \cdots\times\R^{n_m}$
obtained from the completion of $\us$ with zero entries. Then, since the adjoint operator 
$\Ss^*_{(k,L)}\colon \R^{n_1}\times \cdots\times\R^{n_m} \to \R^L$ is the standard projection, we can define 
\begin{equation}
\label{eq:Rkloper} 
\M_{(k,L)} := \Ss^*_{(k,L)} \M \colon \Xx \to \R^L,
\end{equation}
and the set
\begin{equation}
\label{eq:transpoly2}
\begin{aligned}
\Pi_{(k,L)} :&= \{\Xs \in \Xx \,\vert\, \Ss_{(k,L)}^* \M(\xs) = \Ss_{(k,L)}^*(\as_1,\dots, \as_m)\}\\
&= \{\Xs \in \Xx \,\vert\,  \M_{(k,L)}(\xs) = {\as_k}_{\lvert L}\}\\
&= \{\Xs \in \Xx \,\vert\,  (\M_k(\xs))_{\lvert L} = {\as_k}_{\lvert L}\}.
\end{aligned}
\end{equation}
Note that in the definition of $\Pi_{(k,L)}$
we ask for the $k$-th marginal of $\Xs$ to be equal
to $\as_k$ only on the components in $L$. 
Now, given $\tau = (\tau_k)_{1 \leq k \leq m}$ a vector of batch sizes, we  set
\begin{equation*}
\Ic(\tau) = \{(k,L)\,\vert\,k\in[m], L\subset [n_k]\;\vert\; |L|\leq\tau_k\}    
\end{equation*}
and obtain $\Pi = \bigcap_{(k,L)\in\Ic(\tau)} \Pi_{(k,L)}$.

Hence, the \emph{iterative Bregman projections} (IBP) algorithm, for problem \eqref{eq:RMOT_BP}, leads to the following procedure. Let $\xs^0 = $ $\xi \odot \otimes_{k=1}^m \as_k = e^{-\Cs/\eta} \odot \otimes_{k=1}^m \as_k \in \inte(\dom \Ent) = \Xx_{++}$ and define the sequence $\xs^{t}$ recursively as
follows
\begin{equation}\label{eq:IBP}
\begin{array}{l}
\text{for}\;t=0,1,\ldots\\
\left\lfloor
\begin{array}{l}
\text{choose } (k_t,L_t) \in \Ic(\tau), \\
\xs^{t+1} = \Pc_{\Pi_{(k_t,L_t)}}(\xs^{t}).
\end{array}
\right.
\end{array}
\end{equation}
Since the generalized Pythagora's theorem for Bregman projections
(see equation \eqref{eq:Pytha} in Appendix~\ref{app:bregman}) yields that $\KL_{\Pi}(\xs^t) 
= \KL_{\Pi_{(k,L)}}(\xs^t) + \KL_{\Pi}(\Pc_{\Pi_{(k,L)}}(\xs^t))$, in \eqref{eq:IBP} one may choose the sets in a \emph{greedy} manner as 
\begin{equation}\label{eq:greedy}
(k_t,L_t) = \argmax_{(k,L)\in\Ic(\tau)}\KL_{\Pi_{(k,L)}}(\xs^{t}),
\end{equation}
so that
\begin{equation}\label{eq:greedy2}
(k_t,L_t)  = \argmin_{(k,L)\in\Ic(\tau)}\KL_{\Pi}( \Pc_{\Pi_{(k,L)}}(\xs^t) )
\,\text{ and }\, \KL_{\Pi}(\xs^{t+1}) = \min_{(k,L)\in\Ic(\tau)}\KL_{\Pi}( \Pc_{\Pi_{(k,L)}}(\xs^t) ).
\end{equation}
This means that the next iterate is chosen, among the possible projections, 
as the one which is
the closest to the target set $\Pi$. Notable examples of existing algorithms that fit in this framework are greedy multimarginal Sinkhorn of 
\cite{LHCJ2020} ($\tau_k = n_k$) and bi-marginal Greenkhorn of \cite{AWR2017} ($m=2$, $\tau_1 = \tau_2 = 1$). We emphasize that 
this greedy strategy
typically leads to the best performance, provided that it can be implemented efficiently. 
In the following proposition and subsequent remak we show that
the projection onto the sets 
$\Pi_{(k,L)}$ can be computed in a closed form and that the greedy choice of the sets $\Pi_{(k,L)}$'s can indeed be implemented efficiently. 
The proof is postponed in Appendix~\ref{app:alg}.

\begin{restatable}{proposition}{propKLproj}\label{prop:KL_proj}
For every $\Xs\in \Xx_+$, $k\in[m]$ and $L\subset[n_k]$,  
\begin{equation}\label{eq:KL_proj_def}
\Pc_{\Pi_{(k,L)}}(\Xs) = 
\nabla\Ent^*\big( \nabla\Ent(\pi) + \M_{(k,L)}^*(\bar{\us}) \big) 
= \pi \odot \exp \big(\M^*_{(k,L)}(\bar{\us}) \big),
\end{equation}
where
\begin{equation}\label{eq:dual_param}
\bar{\us} = \argmin_{\us\in\R^L} 
\Ent^*\big( \nabla\Ent(\pi) + \M^*_{(k,L)}(\us) \big) 
- \Ent^*( \nabla\Ent(\pi)) - \scalarp{{\as_k}_{\lvert L},\us } 
= \log \frac{ {\as_k}_{\lvert L} }{{\M_k(\pi)}_{\lvert L} },
\end{equation}
and, consequently, for every $j \in \mathcal{J}$,
\begin{equation}\label{eq:KL_proj}
(\Pc_{\Pi_{(k,L)}}(\Xs))_{j} = 
\xs_{j} \times
\begin{cases}
\dfrac{a_{k,j_k}}{{\M_{k}(\xs)}_{j_k}} &\text{if } j_k \in L,  \\[2ex]
1 & \text{otherwise}.
\end{cases}
\end{equation}
Moreover,
\begin{equation}\label{eq:KL_itstep}
\KL_{\Pi_{(k,L)}} (\xs)
=\KL\big( \Pc_{\Pi_{(k,L)}} (\Xs),\Xs) 
= \mathsf{KL} ({\as_k}_{\lvert L}, \M_k(\xs)_{\lvert L}).
\end{equation}
\end{restatable}

\begin{remark}\label{rmk:20211203a}
\normalfont
It follows from \eqref{eq:KL_itstep} that the greedy choice described above
can be implemented by computing $m$ vectors of sizes 
$n_k$ and then choosing $k_t$ among $m$ as the index of the vector that has the maximal sum of the largest $\tau_k$
components. More formally one 
let $\mathsf{d}_k = (\KL(a_{k,1}, (\M_k(\xs^t))_1), \dots, 
\KL(a_{k,n_k}, (\M_k(\xs^t))_{n_k})) \in \R^{n_k}$ and consider the vector 
$\mathsf{d}_{k\downarrow} \in \R^{n_k}$ which has the components of $\mathsf{d}_k$ arranged in a decreasing order. Then $k_t = \argmax_{k \in [m]} \big(\sum_{j_k=1}^{\tau_k} (\mathsf{d}_{k\downarrow})_{j_k} \big)$ and $L_t$ corresponds to the indexes of the largest $\tau_{k_t}$ components of $\mathsf{d}_{k_t}$.
\end{remark}

To ensure better numerical stability, especially for small regularization parameters $\reg$, it is more convenient to work with the dual variables $\nabla\Ent(\Xs) = \log(\Xs)$. More precisely, according to \eqref{eq:KL_proj_def} we have that each iteration of the IBP algorithm \eqref{eq:IBP} can be parameterized as 
\begin{equation}
\label{eq:primaldual}
\Xs^{t} = \exp\Big( - \Cs / \reg + \bigoplus_{k=1}^m \vs_k^t  \Big)
\odot\bigotimes_{k=1}^m \as_k,\;t\in \N,
\end{equation}
and one can implement the algorithm by updating only the dual variables $\vs_k^{t} = (v^t_{k,j})_{1 \leq j\leq n_k} \in\R^{n_k}$, $k\in[m]$, 
also known as \emph{potentials} (see Proposition~\ref{prop:primaldual}).
Thus, in the end, the IBP algorithm \eqref{eq:IBP}-\eqref{eq:greedy} can be written as the following Algorithm \ref{alg}.
 
\begin{algorithm}[ht!]\label{alg}  \caption{$\mathtt{BatchGreenkhorn}(\as_1,\dots,\as_m,\Cs,\reg,\tau)$}
\SetAlgoLined
\KwIn{$(\as_1,\dots,\as_m)$, $\Cs\in\Xx_+$, $\reg>0$, $\tau=(\tau_1,\dots, \tau_m)$, $1 \leq \tau_k \leq n_k$}\vspace{.1truecm}
\textbf{Initialization:} $t=0$, $\vs^{t} _k = 0$, $k\in[m]$, $\mathsf{r}_k^{t} = \M_k(\exp( - \Cs/\reg +\oplus_{k=1}^m\vs_k^t) \odot\otimes_{k=1}^m \as_k)$ \\ \vspace{.1truecm}
\For{
$t=0,1,\ldots$
}{
Compute $(k_t,L_t) = \argmax_{(k,L)\in\Ic(\tau)} \KL({\as_{k}}_{\lvert L},{\mathsf{r}^t_{k}}_{\lvert L})$\\ \vspace{.1truecm}
Set $\vs_k^{t+1} = \vs_k^{t}$, $k\in[m]$ and update 
${\vs_{k_t}^{t+1}}_{\lvert L_t} \gets {\vs^{t+1}_{k_t}}_{\!\!\!\lvert L_t} + \log({\as_{k_t}}_{\lvert L_t}) - \log({\mathsf{r}_{k_t}^{t}}_{\lvert L_t})$\\ 
For $k \in[m]$ compute $\mathsf{r}_{k}^{t+1} 
= \M_k (\exp( -\Cs / \reg + \oplus_{k=1}^m \mathsf{v}_k^{t+1}) \odot\oplus_{k=1}^m \as_k) $ \\ \vspace{.1truecm}
}
\KwOut{$\pi^t = \exp(-\Cs/\eta + 
\oplus_{k=1}^m \mathsf{v}_k^{t}
) \odot \otimes_{k=1}^m \as_k$
}
\end{algorithm}

\section{Convergence theory for Batch Greenkhorn algorithm}\label{sec:conv}

Results on the convergence of general IBP are typically without any rates \cite{BRE1967,CL1981,CR1996}, with the notable exception of \cite{KS2021} where the explicit local linear rate for the greedy IBP was derived. In the following  we prove global linear convergence and derive the explicit dependence of the rate on the given data in two important cases. Moreover, we also provide an analysis of the iteration complexity of Algorithm~\ref{alg}.

Based on the properties of operators $\M$ and $\M_{(k,L)}$ we can derive the main results using the properties of KL as Bregman divergence. The proofs are given in Appendix~\ref{app:conv}. 

\begin{restatable}{theorem}{thmBGrate}\label{thm:rate_global} 
Algorithm~\ref{alg} converges linearly. More precisely, if $(\vs_k^{t})_{k \in [m]}$ are generated by Algorithm \ref{alg}, then the primal iterates given by \eqref{eq:primaldual} converge linearly in KL divergence to $\pi^\star$ given by \eqref{eq:RMOT_BP}, i.e.,
\begin{equation}\label{eq:KL_global}
(\forall\,t\in \N)\quad\KL(\Xs^\star,\Xs^{t}) \leq 
\bigg(1 - \frac{e^{- (2\norm{\Cs}_\infty / \eta + 3M_1)}}{b_\tau - 1} \bigg)^t 
\KL(\Xs^\star,\Xs^0),
\end{equation}
where $b_\tau = \sum_{k\in[m]} \lceil n_k / \tau_k \rceil $, and $0<M_1 <+\infty$ is a constant independent of the batch sizes that satisfies 
$\norm{\bigoplus_{k=1}^m \vs^\star_k}_{\infty}, \norm{\bigoplus_{k=1}^m \vs^t_k}_{\infty} \leq M_1$
for $t \in \N$.
\end{restatable}

\begin{restatable}{theorem}{thmBGiterations}\label{thm:rate_iterations}
Let $\varepsilon>0$ and suppose that $\eta > \varepsilon$.
For Algorithm~\ref{alg}, the number of iterations required to reach the stopping criterion $d_t:=\max_{k\in[m]}\norm{\as_k-\M_k(\xs^t)}_1\leq \varepsilon$ satisfies
\begin{equation}\label{eq:required_iter}
t \leq 2 + \max_{k\in[m]} \bigg\lceil \frac{n_k}{\tau_k} \bigg\rceil\frac{5 M_2}{\varepsilon} 
(2 + M_2 \eta),
\end{equation}
where $0<M_2<+\infty$ is a constant independent of the batch sizes such that $\sum_{k\in[m]}\norm{\vs_k^\star-\vs_k^t}\leq M_2$, for all $t \in \N$.
\end{restatable}

\begin{remark}
\label{rmk:20211201e}
\normalfont
We stress that the constants $M_1$ and $M_2$ considered in the above theorems
always exist (see the proofs in the Appendix~\ref{app:conv}), but
we could not have obtained 
general explicit expression for them depending on the problem data,
which are valid for any $m$ and $(\tau_k)_{k \in [m]}$.
On the other hand,
in the following, we show the important cases $m=2$ and $(\tau_k)_{k \in [m]}$ arbitrary
and $m>2$ and $(\tau_k)_{k \in [m]} = (n_k)_{k \in [m]}$ for which 
we do provide explicit dependence on the problem data.
\end{remark}

A natural issue for the \texttt{BatchGreenkhorn} algorithm is that of optimizing the batch size. 
As extreme cases we have full batch ($\tau_k = n_k$), 
which yields the (greedy) MultiSinkhorn algorithm proposed in \cite{LHCJ2020}, and $\tau_k = 1$, which, in the bi-marginal case, is known as the Greenkhorn algorithm \cite{AWR2017}. In this respect we observe that, since step 3 in Algorithm~\ref{alg} can be efficiently implemented, as discussed in Remark~\ref{rmk:20211203a}, the largest computational cost lies in 
the computation of the marginals $\mathsf{r}_{k}^{t+1}$ in step 5. For simplicity, let us assume that $n_k = n$ and $\tau_k = \tau$,
 for every $k \in [m]$. Then, computing it naively yields $\bigO(m n^{m})$ operations, but indeed it can be done more efficiently in 
 $\bigO(\tau n^{m-1})$ as we show in Appendix~\ref{app:alg}. This way $n / \tau$ iterations with batch size $\tau$ have the same computational cost 
 of one iteration with a full batch $n$ and consequently 
 $b_\tau = m n/\tau$ iterations of \textsc{BatchGreenkhorn} with batch sizes $\tau$
 corresponds to one cycle of cyclic multimarginal Sinkhorn.
Hence, the number of normalized cycles $T = t/b_\tau$ required to satisfy 
the stopping criterion given in Theorem~\ref{thm:rate_iterations} is
\begin{equation}
\label{eq:20211201h}
T \leq 1 + \frac{5 M_2}{m \varepsilon}(2 + M_2),
\end{equation}
which we stress is independent on the batch-sizes and the dimension $n$. 
Whereas, the rate \eqref{eq:KL_global} w.r.t the normalized cycles becomes 
\begin{equation}
\KL(\Xs^\star,\Xs^{b_\tau T}) \leq \bigg[\bigg(1 - \frac{e^{- (2\norm{\Cs}_\infty / \eta + 3M_1)}}{b_\tau - 1} \bigg)^{b_\tau}\bigg]^T \KL(\Xs^\star,\Xs^0).
\end{equation}

We now introduce the analysis of the special cases discussed in Remark~\ref{rmk:20211201e}.
The proofs are based on bounding the potentials and are detailed in Appendix~\ref{app:conv}. 

\begin{restatable}{theorem}{thmBGrateBimarginal}\label{thm:rate_global_bi_marginal} If $m=2$, then algorithm $\mathtt{BatchGreenkhorn}(\as_1,\dots,\as_m,\Cs,\reg,\tau)$ converges linearly with the global rate
\begin{equation}\label{eq:KL_global_bi-marginal}
(\forall\,t\in \N)\quad\KL(\Xs^\star,\Xs^{t}) \leq \left(1 -\frac{e^{- 20\norm{\Cs}_{\infty}/\eta } }{b_\tau -1} \right)^t \KL(\Xs^\star,\Xs^0).
\end{equation}
Moreover, when $\eta>\varepsilon$, the number of iterations required to reach the stopping criterion $d_t\leq \varepsilon$ satisfies
\begin{equation}\label{eq:required_iter_bimarginal}
t \leq 2 + \max_{k\in[m]} \bigg\lceil \frac{n_k}{\tau_k} \bigg\rceil   \frac{15 \norm{\Cs}_\infty(2 + 3 \norm{\Cs}_\infty) }{\eta \varepsilon}.
\end{equation}
\end{restatable}

The following theorem provides new insights into a known algorithm.

\begin{restatable}{theorem}{thmBGrateMultimarginal}
\label{thm:rate_global_multi-marginal} If for all $k\in[m]$ $\tau_k=n_k$, then algorithm $\mathtt{BatchGreenkhorn}(\as_1,\dots,\as_m,\Cs,\reg,\tau)$, i.e. MultiSinkhorn algorithm of \cite{LHCJ2020}, converges linearly with the global rate
\begin{equation}\label{eq:KL_global_multi-marginal}
(\forall\,t\in \N)\quad\KL(\Xs^\star,\Xs^{t}) \leq 
\left(1 - \frac{e^{- (12m-7) \norm{\Cs}_\infty / \eta}}{m-1} \right)^{t} 
\KL(\Xs^\star,\Xs^0).
\end{equation}
Moreover, the number of iterations required to reach stopping criterion $d_t\leq \varepsilon$ satisfies
\begin{equation}\label{eq:required_iter_mot}
t \leq 1 + \frac{8(4m-3) \norm{\Cs}_\infty }{\eta \varepsilon}.
\end{equation}
\end{restatable}

\begin{remark}\label{rm:gredy_vs_cyclic}
\normalfont
Concerning the linear convergence rate, we notice that when the batch is full
($\tau_k=n_k$, $k\in[m]$), cyclic Sinkhorn of \cite{BCCNP2015, Car2021} and (greedy) MultiSinkhorn algorithm of \cite{LHCJ2020} generally differ (unless $m=2$). Moreover, our results shows (for the first time) that the rate of convergence of (greedy) MultiSinkhorn algorithm is strictly better than that of the 
cyclic Sinkhorn algorithm obtained in \cite{Car2021}. Indeed in \cite{Car2021},
the following rate was provided
\begin{equation}
\label{eq:20211201g}
(\forall\,T\in \N)\quad
\KL(\Xs^\star,\Xs^{mT}) \leq 
\left(1 - \frac{e^{- 8(2m-1) \norm{\Cs}_\infty / \eta}}{m}  \right)^{T} 
\KL(\Xs^\star,\Xs^0),
\end{equation}
where $T$ counts the number of cycles, each one consisting of $m$ $\mathsf{KL}$
projections on the given marginals.
Whereas, it follows from our rate \eqref{eq:KL_global_multi-marginal} that
for the (greedy) MultiSinkhorn algorithm we have
\begin{equation}
(\forall\,T\in \N)\quad
\KL(\Xs^\star,\Xs^{mT}) \leq 
\bigg[\bigg(1 - \frac{e^{- (12m-7) \norm{\Cs}_\infty / \eta}}{m-1} \bigg)^{m} \bigg]^T
\KL(\Xs^\star,\Xs^0),
\end{equation}
which mainly gains an $m$-root in the rate.
Note that in the bi-marginal case, the rate of the classical Sinkhorn algorithm according to \cite{Car2021}  is $1- (1/2) e^{-24  \norm{\Cs}_\infty / \eta }$, while we obtain $(1-  e^{-17  \norm{\Cs}_\infty / \eta })^2$.
\end{remark}

\begin{remark}\label{rm:complexity}
\normalfont
Concerning iteration complexity, we note that in literature
the stopping criteria concerns $\sum_{k \in [m]} \norm{\as_k - \M_k(\xs^t)}_1$
and the assumptions usually demand 
$n_k = n$ and $\tau_k = \tau$, for every $k \in [m]$. 
In this setting,
Theorems~\ref{thm:rate_global_bi_marginal}
and \ref{thm:rate_global_multi-marginal} provide the following bounds (in terms of normalized cycles)
\begin{equation}
T \leq 1 +   \frac{15 \norm{\Cs}_\infty(2 + 3 \norm{\Cs}_\infty) }{\eta \varepsilon}
\quad\text{and}\quad
T \leq 1 + \frac{8(4m-3) \norm{\Cs}_\infty }{\eta \varepsilon}.
\end{equation}
respectively.
Those bounds improve existing results from \cite{LHCJ2020}, related to (greedy) MultiSinkhorn, and \cite{LHJ2020}, 
for bi-marginal Sinkhorn and Greenkhorn. Indeed, defining $a_{\min} = \min_{k \in [m], j \in [n]} a_{k,j}$, in \cite{DGK2018,LHCJ2020}, the bi-marginal Sinkhorn and Greenkhorn algorithm are shown to feature the following iteration complexity
\begin{equation*}
T \leq 1 + \frac{2 ( \norm{\Cs}_\infty/\eta + \log (a^{-1}_{\min}))}{\varepsilon}
\quad\text{and}\quad
T \leq 1 + \frac{56 ( \norm{\Cs}_\infty/\eta + \log(n)+2\log (a^{-1}_{\min}))}{\varepsilon}
\end{equation*}
whereas, in \cite{LHJ2020}, for the MultiSinkhorn algorithm, the following bound is provided
\begin{equation}
T \leq 1 + \frac{2 m ( \norm{\Cs}_\infty/\eta + \log (1/a_{\min}))}{\varepsilon}.
\end{equation}
All those bounds contain the term $\log(a^{-1}_{\min})$ which at the best
is $\log (n)$. So they all depend on the dimension of the problem.
Our results remove this dependency.
\end{remark}

\begin{remark}
\normalfont
Here we look at the computational complexity in terms of 
arithmetic operations. To this purpose 
 we let $\bigO_\xi$
be the number of arithmetic operations needed to compute a full marginal
using the Gibbs kernel $\xi$. When we factor this number out of the total computational complexity, what remains is the number of iterations normalized with respect to the full batch, meaning, $\bar{t} = \tau t/n$. Then according to equation \eqref{eq:20211201h}
the total computational complexity is given by
\begin{equation}
\bigg[m + \frac{5 M_2}{\varepsilon}(2 + M_2)\bigg] \bigO_\xi.
\end{equation}
We note that in the worst case $\bigO_\xi$ is of the order of $n^m$,
but this cost can be significantly reduced for structured costs \cite{PC2019}.
\end{remark}

\section{Conclusions, limitations and future work}\label{sec:con}

We present a new algorithm for solving multimarginal entropic regularized
optimal transport problems, called \emph{batch Greenkhorn}, which
is an extension of the popular Greenkhorn algorithm, that can handle 
multiple marginals and at each iteration select in a greedy fashion
a batch of components of a marginal.
We study the convergence of the algorithm in the 
framework of the iterative Bregman projections method,
providing novel linear rate of convergence as well as iteration
complexity bounds. We made a comprehensive comparison with
existing results showing the improvements over the state-of-the-art.
A problem which remains open is that of deriving bounds on the dual variables with 
an explicit dependence on the given problem data for $m\geq3$ when the batch is not full. According to our general Theorems~\ref{thm:rate_global} and \ref{thm:rate_iterations}, this will allow to have 
explicit linear rate and iteration complexity in all possible cases.
Additional research directions are the extension of our analysis to infinite dimensions and general convex regularizers, implementing batch Greenkhorn with structured costs, and analyze the impact of parallel computations.

\newpage
\appendix

\vspace{9ex}
\noindent{\bf \LARGE Appendices
}

\vspace{.3truecm}
   
This supplementary material is organized as follows:
\begin{itemize}
    \item In Appendix~\ref{app:bregman} we provide some basic facts on Bregman divergences.
     \item Appendix \ref{app:alg} contains the proofs of the results stated in Section  \ref{sec:bg}, notably Proposition \ref{prop:KL_proj}, and gives more information on the implementation of Algorithm \ref{alg}.
    \item Finally, in Appendix~\ref{app:conv} we provide the proof of the main results of Sec. \ref{sec:conv}, concerning the linear convergence and iteration complexity of Algorithm \ref{alg}.
\end{itemize}
For the reader's convenience all  results presented in the main body of the paper are restated in this supplementary material.
\appendix

\section{Bregman divergences and Bregman projections}\label{app:bregman}

In this section we recall few facts on Bregman distance and Bregman projections onto affine sets.
In the following $\X$ is an Euclidean space and $\Leg\colon\X\to\extR$ is an extended-real valued function. The set of minimizers of the function $\Leg$
is denoted by $\argmin_{x \in\X} \Leg(x)$, the \emph{domain} of $\phi$  is 
$\dom \Leg :=\{x\in\X \,\vert\, \Leg(x)<+\infty\}$ and $\phi$ is \emph{proper} when $\dom \Leg \neq \varnothing$. 
The function $\Leg$ is \emph{convex} if $\Leg(t x+(1-t)y)\leq t\Leg(x) + (1-t) \Leg(y)$ for all $x,y\in \dom\Leg$ and
$t\in[0,1]$. If the above inequality is strict when $0<t<1$ and $x \neq y$, the function is \emph{strictly convex}. 
The function $\Leg$ is \emph{closed} if the sublevel sets $\{x\in\X\,\vert\, \Leg(x)\leq t\}$ are closed in $\X$ for 
any $t\in\R$. For a convex function $\Leg\colon\X\to\extR$, we denote by $\Leg^{*}$ its \emph{Fenchel conjugate}, 
that is, $\Leg^*\colon\X\to\extR$, $\Leg^{*}(y):=\sup_{x\in\X}\{\scalarp{x,y}-\Leg(x)\}$. The conjugate of a convex
function is always closed and convex, and if $\Leg$ is proper closed and convex, then $(\Leg^{*})^{*}=\Leg$. 

A proper closed and convex function $\Leg$ is \emph{essentially smooth} if it is differentiable on 
$\mathrm{int}(\dom \Leg) \neq \varnothing$, and $\norm{\nabla \Leg(x_n)} \to +\infty$ whenever 
$x_n \in \inte(\dom \Leg)$ and $x_n \to x \in \mathrm{bdry}(\dom \Leg)$.  The function $\Leg$ is 
\emph{essentially strictly convex} if $\inte(\dom \Leg^*)\neq \varnothing$ and is strictly convex on every 
convex subset of $\dom \partial \Leg$. A \emph{Legendre} function is a proper closed and convex function 
which is also essentially smooth and essentially strictly convex. A function is Legendre if and only if its conjugate is so.
Moreover, if $\Leg$ is a Legendre function, then $\nabla \Leg\colon \inte(\dom \Leg) \to \inte(\dom \Leg^*)$ and 
$\nabla \Leg^*\colon \inte(\dom \Leg^*) \to \inte(\dom \Leg)$ are bijective, inverses of each other, and continuous.  Given a Legendre function $\Leg$, the \emph{Bregman distance} associated to 
$\Leg$ is the function $D_\Leg\colon\X \times\X \to [0,+\infty]$ such that
\begin{equation}
\label{Bdiv}
D_\Leg(x,y) =
\begin{cases}
\Leg(x)-\Leg(y)-\scalarp{x-y,\nabla \Leg(y)}&\text{if } y \in \inte(\dom \Leg)\\
+\infty &\text{otherwise}.
\end{cases}
\end{equation}

\begin{fact}
\label{fact}
Let $\Leg$ be a Legendre function on $\X$. Then the following hold
\begin{enumerate}[label={\rm (\roman*)}]
\item\label{fact_i} $(\forall\, \xs,\ys \in \dom \Leg)\quad
D_\Leg(\xs,\ys) + D_\Leg(\ys,\xs) = \scalarp{\xs-\ys, \nabla \Leg(\xs) - \nabla \Leg(\ys)}$.
\item\label{fact_ii} $(\forall\, \xs,\ys \in \dom \Leg)\quad
D_\Leg(\xs,\ys) = D_{\Leg^*}(\nabla\Leg(\ys),\nabla\Leg(\xs))$.
\item\label{fact_iii} If $\Leg$ is twice differentiable, then\\
$(\forall\, \xs,\ys \in \dom \Leg)(\exists\, \xi \in [\xs,\ys])\qquad D_\Leg(\xs,\ys) 
= \dfrac 1 2 \scalarp{\nabla^2 \Leg(\xi)(\xs - \ys), \xs - \ys}$,\\[1ex]
where $[\pi,\gamma]= \{(1-\alpha) \pi + \alpha \gamma \,\vert\, \alpha \in [0,1]\}$
is the segment with end points $\pi$ and $\gamma$.
\item\label{fact_iv} Suppose that $\Leg$ is twice differentiable on $\inte(\dom\Leg)$. Then
\begin{equation}
\big(\forall \xs\!\in\inte(\dom\Leg),\, \nabla^2\Leg(\xs)\text{ is invertible}\big)\Leftrightarrow
\big(\Leg^*\!\text{ is twice differentiable}\big).
\end{equation}
\end{enumerate}
\end{fact}

\begin{fact}
\label{fact1}
Let $\Leg$ be a Legendre function on $\X$ and $\Leg^*$ be it's Fenchel conjugate. If $\dom\Leg^*$ is open, then the following hold
\begin{enumerate}[label={\rm (\roman*)}]
\item\label{fact1_i} For every $\xs \in \inte(\dom \Leg)$,  
the sublevel sets of $D_\Leg(\xs,\cdot)$ are compact.
\item\label{fact1_ii} For every $\xs \in \inte(\dom \Leg)$,  
and every sequence $(\ys_k)_{k \in \N}$ in $\inte(\dom \Leg)$
\begin{equation}
D_\Leg(\xs,\ys_k) \to 0\ \Rightarrow\  \ys_k \to \xs.
\end{equation}
\end{enumerate}
\end{fact}

Let $\Cc \subset \X$ be an affine set, represented
as follows
\begin{equation}
\label{eq:affineset}
A\colon \X\to\Y,\qquad b\in\range(A),\qquad \Cc := \{\xs\in\X \,\vert\, A\xs = b \},
\end{equation}
for some linear operator $A$ between $\X$ and another Euclidean space $\Y$.
Given a Legendre function $\Leg\colon \X \to \left]-\infty,+\infty\right]$ and $\xs \in \inte(\dom \Leg)$, the \emph{Bregman projection} of $\xs$ onto $\Cc$ is defined as
the unique solution, denoted by  $\Pc^\Leg_\Cc(\xs)$, of the optimization problem
\begin{equation}
\min_{\ys\in\Cc}D_\Leg(\ys,\xs) = \min_{\ys\in\Cc} \Leg(\ys) - \Leg(\xs) - \scalarp{\ys-\xs, \nabla \Leg(\xs)}
\end{equation}
and the optimal value defines the \emph{Bregman distance from $\xs$ to $\Cc$} 
and is denoted by $D_{\Cc}^{\Leg}(\xs)$.
The dual of the above problem is 
\begin{equation}\label{eq:BP_dual}
\min_{\lambda \in \Y}  \Leg^*(\nabla\Leg(\xs) + A^* \lambda) - \Leg^*(\nabla\Leg(\xs)) - \scalarp{b,\lambda}
\end{equation}
and strong duality holds, meaning that
\begin{equation}
\min_{\ys\in\Cc}D_\Leg(\ys,\xs) = - \min_{\lambda \in \Y} \big[ \Leg^*(\nabla\Leg(\xs) + A^* \lambda) - \Leg^*(\nabla\Leg(\xs)) - \scalarp{b,\lambda} \big].
\end{equation}
Moreover, the following KKT conditions hold for 
a couple $( \Xs^\star,\lambda^\star)$ solving the primal and dual problem above
\begin{equation}
\label{eq:kkt}
\pi^\star \in \inte(\dom \Leg),\quad
A \Xs^\star = b \quad\text{and} \quad \nabla\Leg(\xs) + A^* \lambda^\star = \nabla \Leg( \Xs^\star).
\end{equation}
Note that the KKT conditions characterizes the projection, so that
\begin{equation}
\label{eq:kkt2}
\Xs^\star = \Pc^\Leg_\Cc(\xs)
\ \Leftrightarrow\ 
\big(\pi^\star \in \inte(\dom \Leg),\ 
A \Xs^\star = b, \ \text{and} \ \nabla\Leg(\xs) - \nabla \Leg( \Xs^\star) \in \range(A^*) \big).
\end{equation}

Finally we mention the \emph{generalized Pythagora's theorem}. If $\Cc_1$ is an affine set such that $\Cc \subset \Cc_1$, then, for every $\xs\in\inte(\dom\Leg)$ it holds
\begin{equation}
\label{eq:Pytha}
D_\Cc^{\Leg}(\xs) = D_{\Cc_1}^\Leg(\xs) + D_{\Cc}^\Leg(\Pc^\Leg_{\Cc_1}(\xs)).
\end{equation}
Moreover, in this case  $\Pc^\Leg_\Cc(\xs)=\Pc^\Leg_\Cc(\Pc^\Leg_{\Cc_1}(\xs))$, and 
\begin{equation}\label{eq:same_proj}
 (\forall \gamma\in\inte(\dom\Leg)) \quad \nabla\Leg(\gamma) - \nabla\Leg(\pi)\in\range(A^*) \Longleftrightarrow \Pc_\Cc(\gamma) = \Pc_\Cc(\pi).
\end{equation}


In the following we let $\phi\colon \R \to \left]-\infty,+\infty\right[$ be the (negative)
\emph{Boltzmann-Shannon entropy}, that is,
\begin{equation*}
\phi (t) = 
\begin{cases}
t\log t - t &\text{if } t > 0\\
0 &\text{if } t=0\\
+\infty &\text{if } t<0.
\end{cases}
\end{equation*}
It is clear that $\phi^*(s) = \exp(s)$.
We define the Bregman distance associated to $\phi$
\begin{align}
\nonumber
D_\phi(s,t) &= 
\begin{cases}
\phi(s) - \phi(t) - \phi^\prime(t)(s-t) & \text{if } t>0\\
+\infty &\text{otherwise}.
\end{cases}\\[1ex]
\label{eq:KL1}
&=\begin{cases}
s \log \frac{s}{t} - s + t & \text{if } t>0\\
+\infty &\text{otherwise},
\end{cases}
\end{align}
which is nothing but the \emph{Kullback-Leibler divergence} on $\R$.

\begin{proposition}
\label{prop:SGphi}
Let $M>0$. The following hold.
\begin{enumerate}[label={\rm (\roman*)}]
\item\label{prop:SGphi_i} The function $\phi$ is strongly convex on the interval 
$\left]0, M \right]$
with modulus of strong convexity equal to $1/M$. Moreover,
for every $a>0$ and $s,t \in \R_{++}$, $D_\phi(s,t) = a D_\phi(s/a, t/a)$.
\item\label{prop:SGphi_ii} The function $\phi^*$ is strongly convex on the interval 
$\left[ - M, +\infty \right[$ with modulus of strong convexity equal to $\exp(-M)$.
\end{enumerate}
\end{proposition}
\begin{proof}
\ref{prop:SGphi_i}: It follows from the fact that the second derivative of $\phi$
is $\phi^{\prime\prime}(t)= 1/t$, which is bounded from below away from zero
on the interval $\left]0, M \right]$ by the constant $1/M$. the second part follows directly from the definition \eqref{eq:KL1}.

\ref{prop:SGphi_ii}: 
It follows from the fact that the second derivative of $\exp$ is bounded from below away from zero on the interval $\left[ - M, +\infty \right[$ by $\exp(-M)$.
\end{proof}

The negative entropy and the Kullback-Leibler divergence on $\Xx$ are
\begin{equation}
\mathsf{H}(\gamma) = \sum_{j} \phi(\gamma_j)
\quad\text{and}\quad 
\mathsf{KL}(\gamma,\pi) = \sum_{j} D_\phi(\gamma_j, \pi_j).
\end{equation}

\begin{lemma}\label{fact:KLbound}
Let $\pi,\gamma, \alpha\in\Xx_{++}$ and suppose that
\begin{equation*}
0<M_1\leq \min_{j} \frac{\min\{\pi_{j},\gamma_{j}\}}{\alpha_{j}} \leq 
\max_{j} \frac{\max\{\pi_{j}, \gamma_{j}\}}
{\alpha_{j}} \leq M_2.
\end{equation*} 
Then, setting $\As = \alpha\odot(\cdot)\colon\Xx\to\Xx$ (which is a positive diagonal operator), we have
\begin{equation}\label{eq:KLbound}
\KL(\pi,\gamma) \geq \max\left\{\frac{M_1}{2} \norm{\log \pi-\log\gamma}_\As^2,\frac{1}{2 M_2} \norm{\pi-\gamma}_{\As^{-1}}^2\right\}.
\end{equation} 
\end{lemma}
\begin{proof}
It follows from Proposition~\ref{prop:SGphi}\ref{prop:SGphi_i}
that, for $a>0$ and $s,t>0$ such that $s/a,t/a \leq M$, we have
 $D_\phi(s,t) = a D_\phi(s/a,t/a)\geq a (2 M)^{-1} \abs{s/a - t/a}^2 = 
 (2 M)^{-1} a^{-1} \abs{s-t}^2$. Thus, since $\gamma_j/\alpha_j, \pi_j/\alpha_j  \leq M_2$,
 we have
 \begin{equation*}
 \mathsf{KL}(\pi, \gamma)= \sum_{j} D_\phi(\pi_j, \gamma_j)
 \geq \frac{1}{2 M_2} \sum_{j} \frac{1}{\alpha_j} \abs{\pi_j - \gamma_j}^2
 = \frac{1}{2 M_2} \norm{\pi - \gamma}^2_{\As^{-1}}.
\end{equation*}
Now, it follows from Proposition~\ref{prop:SGphi}\ref{prop:SGphi_ii} that
for every $a,s,t>0$ such that $s/a,t/a \geq e^{-M}$ we have $\log (s/a), \log (t/a) \geq - M$ and hence $D_\phi(s,t) = a D_\phi(s/a,t/a) = a D_{\phi^*}(\log (s/a), \log (t/a)) 
\geq a (e^{-M}/2) \abs{\log s - \log t}^2$. Therefore, since $\pi_j/\alpha_j, \gamma_j/\alpha_j \geq M_1$, we have
 \begin{equation*}
 \mathsf{KL}(\pi, \gamma)= \sum_{j} D_\phi(\pi_j, \gamma_j)
 \geq \frac{M_1}{2} \sum_{j} \alpha_j \abs{\log \pi_j - \log \gamma_j}^2
 = \frac{M_1}{2} \norm{\log \pi - \log \gamma}_{\As}^2.
 \qedhere
\end{equation*}
\end{proof}

\section{BatchGreenkhorn algorithm and its implementation}\label{app:alg}

Here we provide proofs of the results in Sec.~\ref{sec:bg}.

\propKLproj*
\begin{proof}
It follows from \eqref{eq:transpoly2} that
\begin{equation}
\Pi_{(k,L)} 
= \Big\{ \xs \in \Xx \,\big\vert\, 
\M_{(k,L)}(\xs) = \as_{k\lvert L} \Big\}.
\end{equation}
Then the first equality in \eqref{eq:KL_proj_def} follows directly form the KKT conditions \eqref{eq:kkt},
which in this case yields
\begin{equation}
\label{eq:20210908a}
    \M_{(k,L)}(\bar{\xs}) = \as_{k\lvert L}
    \quad\text{and}\quad
    \bar{\xs} = \nabla \Ent^* \big(\nabla \Ent(\xs) + \M^*_{(k,L)} (\bar{\us})\big),
\end{equation}
where, according to \eqref{eq:BP_dual}, the dual parameter $\bar{\us} \in \R^L$  solves the minimization problem in  \eqref{eq:dual_param}. Now, since for every $j \in \mathcal{J}$,  $(\nabla \Ent(\xs))_{j} = \log(\xs_{j})$ and $(\nabla \Ent^*(\gamma))_{j} = \exp(\gamma_{j})$, then \eqref{eq:20210908a} gives
\begin{equation*}
    \bar{\xs}_{j} = \exp ( \log (\xs_{j}) + \M^*_{(k,L)}(\bar{\us}))_{j})
    = \xs_{j}  \exp\big( (\M^*_{(k,L)}(\bar{\us}))_{j}\big)
\end{equation*}
and the second equality in \eqref{eq:KL_proj_def} follows.

Now, let $J^{(k)}_L\colon \R^L \to \R^{n_k}$ be the canonical injection of $\R^L$ into $\R^{n_k}$.
Then, recalling the definition of $\M_{(k,L)}$ in \eqref{eq:Rkloper}, we have
$\M_{(k,L)} = J^{(k)*}_L \M_k$ and hence $\M^*_{(k,L)} = \M_k^* J_L^{(k)}$,
where  $\M_k^* \colon \R^{n_k} \to \Xx$ acts as
$(\M_k^* \vs)_{j} = \vs_{j_k}$. Therefore, for every $j \in \mathcal{J}$,
\begin{equation*}
    (\M^*_{(k,L)} \bar{\us})_{j} =
    (\M_k^* J^{(k)}_L \bar{\us})_{j}
    = (J^{(k)}_L \bar{\us})_{j_k}
    = \begin{cases}
    \bar{u}_{j_k} &\text{if } j_k \in L\\
    0 &\text{otherwise}.
    \end{cases}
\end{equation*}
Hence,
\begin{equation}
\label{eq:20210603a}
\bar{\xs}_{j} = \xs_{j}
\times
\begin{cases}
e^{\bar{u}_{j_k}} & \text{if } j_k \in L\\
1 & \text{otherwise}.
\end{cases}
\end{equation}
On the other hand, 
since $\as_{k\lvert L} = J^{(k)*}_L \M_k \bar{\xs}$,
by \eqref{eq:20210603a}, we derive that, for every $j_k \in L$, 
\begin{equation}
\label{eq:20210604a}
a_{k,j_k} = (\M_k \bar{\xs})_{j_k} = 
\sum_{j_{-k} \in \mathcal{J}_{-k}}
\bar{\xs}_{(j_{-k}, j_k)}
= e^{\bar{u}_j} (\M_k \xs)_{j},
\end{equation}
so that $e^{\bar{u}_{j_k}} = a_{k,j_k}/(\M_k \xs)_{j_k}$.
Hence now \eqref{eq:KL_proj} follows from \eqref{eq:20210603a}. Concerning the formula for the distance, by \eqref{eq:KL_proj}, we have that,
\begin{align*}
    D^\Leg_{\Pi_{(k,L)}} (\xs)
    &= D^\Leg(\bar{\xs},\xs)\\
    &= \sum_{j \in \mathcal{J}}
    \bar{\xs}_{j} \log \Big(\frac{\bar{\xs}_{j}}{\xs_{j}} \Big) - \bar{\xs}_{j} 
    + \xs_{j}\\
    & = \sum_{j_k \in L} 
    \sum_{j_{-k} \in \mathcal{J}_{-k}}
    \bar{\xs}_{(j_{-k}, j_k)}
    \log \Big(\frac{a_{k,j_k}}{(\M_k \xs)_{j_k}} \Big) - \bar{\xs}_{(j_{-k}, j_k)} 
    + \xs_{(j_{-k}, j_k)}\\[1ex]
    & = \sum_{j_k \in L} a_{k,j_k} \log \Big(\frac{a_{k,j_k}}{(\M_k \xs)_{j_k}} \Big)
    - a_{k,j_k} + (\M_k \xs)_{j_k}\\[1ex]
    & = \mathsf{KL} ({\as_k}_{\lvert L}, \M_k(\xs)_{\lvert L}),
\end{align*}
which completes the proof.
\end{proof}

The following proposition justifies equation \eqref{eq:primaldual}
and Algorithm~\ref{alg}.
\begin{proposition}\label{prop:primaldual}
Let $(\xs^t)_{t \in \N}$ be defined according to algorithm \eqref{eq:IBP}. Then,
we have
\begin{equation}
\label{eq:20211108b}
(\forall\, t \in \N)
\qquad
    \xs^t
    = \exp \Big( - \Cs/\eta + \bigoplus_{k=1}^m \vs^t_k \Big) 
    \odot \bigotimes_{k=1}^m \as_k,
\end{equation}
where $\vs^t_k = (v^t_{k,j})_{1 \leq j \leq n_k} \in \R^{n_k}$ and
\begin{equation}\label{eq:vupdate}
    \vs_k^{t+1} = \delta_{k,k_t} J_{L_t}^{(k_t)} \us^t + \vs_k^t,
    \quad
    \us^t  = \log \as_{k_t\lvert L_t} - \log (\M_{k_t} \xs^t)_{\lvert L_t},
\end{equation}
$J_{L_t}^{(k_t)} \colon \R^{L_t} \to \R^{n_{k_t}}$ is the canonical injection,  $\delta_{k,k_t}$ is the Kronecker symbol, and $\vs^0_k$, $k\in[m]$ are arbitrary.
\end{proposition}
\begin{proof}
By definition of $\xs^{t+1}$ we have
\begin{equation}
    \xs^{t+1} = \Pc_{\Pi_{(k_t, L_t)}}(\xs^{t}),
    \quad \Pi_{(k_t, L_t)} = \big\{\xs \in \Xx \,\vert\,  J_{L_t}^{(k_t)*} \M_{k_t} \xs = 
    \as_{k_t\lvert L_t} \big\}.
\end{equation}
Then it follows from Proposition \ref{prop:KL_proj} that
\begin{equation}
    \nabla \Ent(\xs^{t+1}) = \nabla \Ent(\xs^{t}) + \M_{k_t}^* J^{(k_t)}_{L_t} \us^t
\end{equation}
with $\us^t \in \R^{L_t}$.
Therefore, applying the above equation recursively, we get
\begin{align*}
    \nabla \Ent(\xs^{t}) 
    &= \nabla \Ent(\xs^{0}) + \sum_{s=0}^{t-1} 
    \M_{k_s}^* J^{(k_s)}_{L_s} \us^s\\
    & = \nabla \Ent(\xs^{0}) 
    + \sum_{s=0}^{t-1} \sum_{k=1}^m 
    \M_{k}^* \delta_{k,k_s} J^{(k_s)}_{L_s} \us^s\\
    & = \nabla \Ent(\xs^{0}) 
    + \sum_{k=1}^m \sum_{s=0}^{t-1} 
    \M_{k}^* \delta_{k,k_s} J^{(k_s)}_{L_s} \us^s\\
    & = \nabla \Ent(\xs^{0}) 
    + \sum_{k=1}^m \M_{k}^* \Big( \sum_{s=0}^{t-1} \delta_{k,k_s}
     J^{(k_s)}_{L_s} \us^s\Big).
\end{align*}
Then if we set $\vs_k^t = \sum_{s=0}^{t-1} \delta_{k,k_s}
     J^{(k_s)}_{L_s} \us^s \in \R^{n_k}$, we have (recalling that $\pi_0 = e^{-\Cs/\eta} \odot \otimes_{k=1}^m \as_k$)
\begin{equation*}
    \xs^t_{j}
    = \exp \Big( \log(\xs^0_j) + \sum_{k=1}^m (\M_k^* \vs^t_k)_{j} \Big)
    = \exp\Big( -\Cs_{j} / \eta + \sum_{k=1}^m v^t_{k,j_k} \Big) \prod_{k=1}^m a_{k,j_k}.
\end{equation*}
Moreover, it is clear that
\begin{equation*}
    \vs_{k}^{t+1} = \sum_{s=0}^{t} \delta_{k,k_s} J_{L_s}^{(k_s)} \us^s
    = \delta_{k,k_t} J_{L_t}^{(k_t)} \us^t + \vs_k^t
\end{equation*}
Finally, it follows from \eqref{eq:dual_param} that, for every $j_{k_t} \in L_t$,
$e^{u_{j_{k_t}}^t} = a_{k_t, j_{k_t}}/(\M_k \xs^t)_{j_{k_t}}$, so that
\begin{equation*}
    (\forall\, j_{k_t} \in L_t)
    \quad u^t_{j_{k_t}} = \log a_{k_t,j_{k_t}} - \log\big( (\M_k \xs^t)_{j_{k_t}} \big).
\end{equation*}
The statement follows.
\end{proof}

Next we give more detailed implementation of the batch Greenkhorn given in Algorithm \ref{alg}.

\begin{remark}[Implementation details on Algorithm~\ref{alg}]
The most delicate part is to avoid recomputing the marginals $\M_k(\pi^t)=\rs_k^t = (r_{k,j_k}^t)_{j_k\in[n_k]}$, $k\in[m]$, (step 5) that are necessary for making the greedy choice in step 3. Now, due to equation \eqref{eq:KL_proj} in Proposition \ref{prop:KL_proj}, we have that 
\begin{equation}
\pi^{t+1}_{j} = 
\xs^t_{j} \times
\begin{cases}
\dfrac{a_{k_t,j_{k_t}}}{{\M_{k_t}(\xs^t)}_{j_{k_t}}} &\text{if } j_{k_t} \in L_t,  \\[2ex]
1 & \text{otherwise}
\end{cases}
\end{equation}
and hence
\begin{equation}
r_{k_t,j_{k_t}}^{t+1} = \begin{cases}
a_{k_t,j_{k_t}} & j_{k_t}\in L_t,\\
r_{k_t,j_{k_t}}^{t} & \text{otherwise}.
\end{cases}
\end{equation}
To derive update formula for the other marginals, observe that for all $k\neq k_t$ it 
follows from \eqref{eq:20211108b} that
\begin{align*}
r_{k,j_k}^{t+1} & = 
\sum_{j_{-k}  \in \mathcal{J}_{-k}}
\exp\Big( \log \pi^0_{(j_{-k}, j_k)}  
+ \sum_{h\neq k} v^{t+1}_{h,j_{h}} + v^{t+1}_{k,j_k}\Big)\\
& = \sum_{j_{-k} \in \mathcal{J}_{-k}} \exp\Big( \log \pi^0_{(j_{-k}, j_k)} + \sum_{h\not\in\{k,k_t\}} v^{t}_{h,j_{h}} + v^{t+1}_{k_t,j_{k_t}} + v^{t}_{k,j_k}  \Big).\\
\end{align*}
So, since according to \eqref{eq:vupdate}, $v^{t+1}_{k_t,j_{k_t}} = v^t_{k_t,j_{k_t}}$ if $j_{k_t}\not\in L_t$ and $v^{t+1}_{k_t,j_{k_t}} = v^t_{k_t,j_{k_t}} + \log(a_{k_t, j_{k_t}}/ r^t_{k_t, j_{k_t}})$ if $j_{k_t} \in L_t$, we have 
\begin{align*}
r_{k,j_k}^{t+1} = &  
\sum_{\substack{j_{-k} \in \mathcal{J}_{-k} \\ j_{k_t}\not\in L_t}} 
\exp\Big( \log \pi^0_{(j_{-k}, j_k)} 
+ \sum_{h \neq k} v^{t}_{h,j_{h}} 
+ v^{t}_{k,j_k}
\Big)\; \\
&\quad + \sum_{\substack{j_{-k} \in \mathcal{J}_{-k} \\ j_{k_t}\in L_t}} 
\exp\Big( \log \pi^0_{(j_{-k}, j_k)} 
+ \sum_{h \neq k} v^{t}_{h,j_{h}} 
+ v^{t}_{k,j_k}
\Big) \frac{a_{k_t,j_{k_t}}}{r^t_{k_t,j_{k_t}}}\\
= & \sum_{\substack{j_{-k} \in \mathcal{J}_{-k} \\ j_{k_t}\in L_t}} 
\exp\Big( \log \pi^0_{(j_{-k}, j_k)} + \sum_{h\neq k} v^{t}_{h,j_{h}} + v^{t}_{k,j_k}\Big) \Big(\frac{a_{k_t,j_{k_t}}}{r^t_{k_t,j_{k_t}}}-1\Big)
 + r_{k,j_k}^{t}.
\end{align*}
Therefore, at each iteration $t\geq0$ we will construct an auxiliary tensor $\widetilde \pi^{t}\in \R_+^{n_1\times\cdots\times n_{k_t-1}\times 1 \times n_{k_t+1} \times\cdots\times n_m}$ by 
\begin{align}
\widetilde\pi^t_{j_1,\ldots,j_{k_t-1},1,j_{k_t+1},\ldots,j_m} =& \displaystyle{\sum_{j_{k_t}\in L_t} \exp\Big( \log \pi^0_{j_1,\ldots,j_{k_t-1},j_{k_t},j_{k_t+1},\ldots,j_m} + \sum_{k\in[m]} v^{t}_{k,j_k}  }\label{eq:pihat1}\\
& \quad + \log \abs{a_{k_t,j_{k_t}} - r^{t}_{k_t,j_{k_t}}} - \log(r^{t}_{k_t,j_{k_t}})\Big) \sgn(a_{k_t,j_{k_t}} - r^{t}_{k_t,j_{k_t}}), \label{eq:pihat2}
\end{align}
in order to obtain that for every $k\neq k_t$, $\rs^{t+1}_k = \rs^{t+1}_k + \M_k(\hat \pi^t)$. Hence, we can use $\hat\pi^t$ to efficiently update non-active marginals without recomputing them from scratch. Moreover, note that using \eqref{eq:pihat1}-\eqref{eq:pihat2} one avoids excessive numerical errors when $a_{k,j}\approx r_{k,j}^t$. These observations lead us to the following implementation of \textsc{BatchGreenkhorn}.

\end{remark}

\medskip

\begin{algorithm}[H]\label{alg1}  \caption{$\mathtt{BatchGreenkhorn}(\Cs,\reg,\rho,\tau,\varepsilon)$}
\SetAlgoLined
\KwIn{$\Cs\in\Xx_+$, $\reg>0$, $(\as_1,\dots,\as_m)$, $(\tau_1,\dots, \tau_m)$, $1 \leq \tau_k \leq n_k$, $\varepsilon>0$} \vspace{.1truecm}
\textbf{Initialization:} $t=0$, $\vs^{0}_k = 0$, $k\in[m]$, $\mathsf{r}_k^{0} = (r^0_{k,j})_{j\in[n_k]}= \M_k(\exp( - \Cs/\reg )\odot\otimes_{k=1}^m\as_k)$  \\ \vspace{.1truecm}

 \While{$\sum_{k\in[m]}\norm{ \as_{k} - \rs^{t}_{k} }_1 >\varepsilon$}{
\For{$k\in[m]$}{
Compute vectors $\ps_k$ as $p_{k,j} :=\KL(a_{k,j},r^{t}_{k,j} )$, for $j\in[n_k]$ \\ \vspace{.1truecm}
Take $L_k'$ to be $\tau_k$ largest elements of $\ps_k$
} 
Choose the marginal with the best batch: $k_t \gets \argmax_{k\in[m]} \norm{{\ps_k}_{\lvert L_k'}}_1$ and $L_t = L_{k_t}'$ \\ \vspace{.1truecm}
Set $\vs_k^{t+1} = \vs_k^{t}$ and update ${\vs_{k_t}^{t+1}}_{\lvert L_t} \gets {\vs^{t+1}_{k_t} }_{\lvert L_t}+ \log({\as_{k_t}}_{\lvert L_t}) - \log({\rs^{t}_{k_t}}_{\lvert L_t})$  \\ \vspace{.1truecm}
Set $\rs_{k_t}^{t+1} = \rs_{k_t}^{t}$ and update ${\rs_{k_t}^{t+1}}_{\lvert L_t} = {\as_{k_t}}_{\lvert L_t}$ \\ \vspace{.1truecm}
For $k \in[m]\setminus\{k_t\}$ update $\rs^{t+1}_{k} \gets \rs^{t}_{k} + \M_k(\widetilde \pi^t) $, where  $\widetilde \pi^{t}$ is given by \eqref{eq:pihat1}--\eqref{eq:pihat2}   \\ \vspace{.1truecm}
Set $t \gets t+1$\\ \vspace{.1truecm}
}
\KwResult{$\{\vs^{t} _k\}_{k\in[m]} $}
\end{algorithm}

\begin{remark}\label{rm:norm_iter}
Let us assume that $\tau_k = \tau$ and $n_k = n$ for all $k\in[m]$ and that $m<<n$. Then, we can conclude that the cost of one iteration of \textsc{BatchGreenkhorn} is essentially determined by step 10 of Algorithm \ref{alg1} which is performed in $\bigO(\tau n^{m-1})$ operations. Hence, one iteration of the \textsc{MultiSinkhorn} (i.e, \textsc{BatchGreenkhorn} with a full batch $\tau=n$) has the same order of computational cost as $n / \tau$ iterations of \textsc{BatchGreenkhorn} with a batch size $\tau$. So, we can introduce the normalized iteration counter as $t = t_{\tau} n / \tau$, where $t_\tau$ is the iteration counter for the \texttt{BatchGreenkhorn} with a batch size $\tau$. 
\end{remark}

\section{Convergence of Batch Greenkhorn algorithm}\label{app:conv}

Here we provide proofs of the main results given in Sec.~\ref{sec:conv}. 
We first set notation for the rest of the section.
Given $k\in[m]$ and $L\subset[n_k]$, 
we denote by 
\begin{equation}
J_k \colon \R^{n_k} \to \R^{n_1}\times\cdots\times\R^{n_k}\times\cdots\times\R^{m},
\quad \vs_k \mapsto (0, \dots, 0, \vs_k, 0, \dots, 0),
\end{equation}
the canonical injection of $\R^{n_k}$ into $\R^{n_1}\times\cdots\times\R^{n_k}\times\cdots\times\R^{n_m}$ and by 
\begin{equation}
J_{L}^{(k)}\colon \R^L \to \R^{n_k}
\end{equation}
the canonical injection of $\R^L$ into $\R^{n_k}$.

We note that, referring to the operators $\M$ and $\M_{k}$ defined in \eqref{eq:Roper} and 
\eqref{eq:pushforw},
respectively, we have
\begin{equation}\label{eq:operators_R}
\M^*\colon \R^{n_1}\times\cdots\times\R^{n_m} \to \Xx,
\qquad \M^*(\vs_1,\dots, \vs_m)=\sum_{k=1}^m \M_k^*(\vs_k)
= \bigotimes_{k=1}^m \vs_k,
\end{equation}
where $(\M_k^*(\vs_k))_{j_1,\dots, j_k,\dots, j_m} = v_{k,j_k}$.
Indeed the second equality in \eqref{eq:operators_R} follows from the fact that
for every $j \in \mathcal{J}$, 
$(\M^*(\vs_1, \dots, \vs_m))_j = \sum_{k=1}^m (\M_k^* \vs_k)_j 
= \sum_{k=1}^m v_{k, j_k} = \big(\bigoplus_{k=1}^m \vs_k \big)_j$.
We note also that $\M_k = J^*_k \circ \M$, since $J_k^*$ is the $k$-th canonical projection.

Then we provide a result concerning the properties of optimal potentials.

\begin{lemma}\label{lm:v_star_bound}
Let $\pi^\star$ be the solution of RMOT given by \eqref{RMOT}. Then $\pi^\star = \Pc_\Pi(\xi\odot\otimes_{k=1}^m\as_k)$ and, for every $k \in [m]$, there exist  
$\vs_k^{\star} = (v^\star_{k,j})_{1 \leq j\leq n_k} \in\R^{n_k}$, such that 
\begin{equation}\label{eq:pi_star}
\Xs^\star = \exp\Big( - \frac{\Cs}{\eta} + \bigoplus_{k=1}^m \vs^\star_k \Big) 
\odot \bigotimes_{k=1}^m \as_k,
\end{equation} 
and the $\vs_k^{\star}$'s, can be chosen so that 
\begin{equation}\label{eq:v_star_bound}
\sum_{k\in[m]}\norm{\vs_k^\star}_\infty \leq (4m-3) \frac{\norm{\Cs}_\infty}{\eta}.
\end{equation} 
Moreover, if $m=2$, then $\vs_1^{\star}$ and $\vs_2^{\star} $ can be chosen such that 
\begin{equation}\label{eq:v_star_bound2}
\max_{k\in[m]}\norm{\vs_k^\star}_\infty \leq \frac{3}{2} \frac{\norm{\Cs}_\infty}{\eta}.
\end{equation} 
\end{lemma}
\begin{proof}
Since, by definition $\xs^\star = \Pc_\Pi(\xi)$,
it easy to see, from the characterization of the projection given in \eqref{eq:kkt2},
that 
\begin{equation*}
\xs^\star = \Pc_\Pi(\xi\odot\otimes_{k=1}^m\as_k)
\ \Leftrightarrow\ 
\nabla\Ent(\xi\odot\otimes_{k=1}^m\as_k) - \nabla\Ent(\xi) \in \range (\M^*).
\end{equation*}
Thus, since $\nabla\Ent(\xi\odot\otimes_{k=1}^m\as_k) - \nabla\Ent(\xi) =   \log \otimes_{k=1}^m\as_k = \oplus_{k=1}^m \log \as_k = \M^*(\log \as_1,\ldots,\log \as_m)\in\range(\M^*)$, we have that $\Pc_{\Pi}(\xi\odot\otimes_{k=1}^m\as_k) = \Pc_{\Pi}(\xi) = \pi^\star$. Now, it follows from the  KKT conditions \eqref{eq:kkt} for the projection of $\xi\odot\otimes_{k=1}^m\as_k$ onto affine set $\Pi$, that
\begin{equation*}
\xs^\star = \nabla\Ent^*\big(\nabla\Ent(\xi\odot\otimes_{k=1}^m\as_k) + \M^*(\vs_1^\star, \dots, \vs_m^\star) \big) 
\end{equation*}
for some $(\vs_1^\star, \dots, \vs_m^\star) \in \R^{n_1}\times \dots\times \R^{n_m}$.
Since $\nabla\Ent^* = \exp$ and $\nabla \Ent = \log$, \eqref{eq:pi_star} follows.
Next, observe that for every $k\in[m]$, since $\M_k(\pi^\star) = \as_k$, using \eqref{eq:pi_star}, we obtain that for every $j_k\in[n_k]$,
\begin{equation}\label{eq:vstar}
\exp(v^{\star}_{k,j_k}) \sum_{j_{-k} \in \mathcal{J}_{-k}} \exp( -\Cs_{(j_{-k},j_k)}/\reg + \sum_{h\neq k}^m v^{\star}_{h,j_{h}} ) \prod_{h\neq k} a_{h,j_{h}} = 1.
\end{equation} 
Hence, the vectors $\vs^\star_1, \ldots, \vs^\star_m$ solve a (discrete) Schr\"{o}dinger system, and we can apply the results from \cite[Lemma~3.1]{Car2021} and \cite[Theorem~2.8]{DG2020} to obtain \eqref{eq:v_star_bound} and \eqref{eq:v_star_bound2}, respectively. 
\end{proof}

\begin{lemma}
\label{lem:20211128a}
Let $\As \colon \Xx \to \Xx$ and $\As_k \colon \R^{n_k} \to \R^{n_k}$ be the diagonal and positive operators defined
as $\As(\Xs) = \Xs \odot \bigotimes_{k=1}^m \as_k$ and
 $\As_k \vs_k = \vs_k \odot \as_k$, respectively.
Let $(\vs_1, \dots, \vs_m) \in \R^{n_1}\times \dots\times \R^{n_m}$.
Then
\begin{equation}
\label{lem:20211128a_eq1}
(\forall\, k \in [m])\quad
\norm{\M_k^* \vs_k}^2_{\As} = \norm{\vs_k}^2_{\As_k}.
\end{equation}
Moreover, if $\scalarp{\vs_k, \as_k}=0$ for every $k=1, \dots, m-1$, then
\begin{equation}
\label{lem:20211128a_eq2}
\norm{\M^*(\vs_1, \dots, \vs_m)}^2_\As = 
\sum_{k=1}^m \norm{\M_k^* \vs_k}^2_{\As} =
\sum_{k=1}^m \norm{\vs_k}^2_{\As_k}.
\end{equation}
\end{lemma}
\begin{proof}
Let $k \in [m]$. Then, recalling that $(\M_k^*(\vs_k))_j = v_{k,j_k}$, we have
\begin{align*}
\norm{\M_k^* \vs_k}^2_{\As} 
&= \sum_{j \in \mathcal{J}} v^2_{k, j_k}
\prod_{\ell=1}^m a_{\ell, j_\ell}\\
& =\sum_{j_{-k} \in \mathcal{J}_{-k}} \sum_{j_k=1}^{n_k}
v^2_{k, j_k} a_{k, j_k}\prod_{\ell\neq k}^m a_{\ell, j_\ell}\\
& = \prod_{\ell \neq k} \bigg(\sum_{\ell=1}^{n_{\ell}} a_{\ell,j_\ell} \bigg)
\sum_{j_k=1}^{n_k} v^2_{k, j_k} a_{k, j_k}.
\end{align*}
Since $\sum_{\ell=1}^{n_{\ell}} a_{\ell,j_\ell}=1$, we get 
$\norm{\M_k^* \vs_k}^2_{\As} = \norm{\vs_k}^2_{\As_k}$ and the first part of the statement follows. Concerning the second part, equation \eqref{lem:20211128a_eq2}
 will follow if we prove that, for every $k, h \in [m]$ with $k \neq h$,
we have that $\M_k^*(\vs_k)$ and $\M_h^*(\vs_h)$ are orthogonal in the metric 
$\scalarp{\cdot, \cdot}_{\As}$.
Thus, let $k, h \in [m]$ and suppose (w.l.o.g.) that $k < h$. Then
\begin{align*}
\scalarp{\M_k^*(\vs_k), \M_h^*(\vs_h) }_{\As} &= \sum_{j \in \mathcal{J}} 
v_{k,j_k} v_{h,j_h} \prod_{\ell=1}^m a_{\ell,j_\ell}\\
& = \sum_{\substack{j_1,\dots,j_{k-1},j_{k+1},\dots,\\j_{h-1},j_{h+1},\dots, j_m}} 
\sum_{j_k=1}^{n_k} \sum_{j_h=1}^{n_h} v_{k,j_k} v_{h,j_h} a_{k, j_k} a_{h, j_h} \prod_{\ell\neq k, \ell\neq h}^m a_{\ell,j_\ell}\\
& =\bigg( \sum_{j_k=1}^{n_k} v_{k, j_k} \bigg)
\bigg( \sum_{j_h=1}^{n_h} v_{h, j_h} \bigg) \prod_{\ell\neq k, \ell\neq h}
\bigg( \sum_{\ell=1}^{n_\ell} a_{\ell, j_\ell} \bigg).
\end{align*}
Since $\sum_{\ell=1}^{n_{\ell}} a_{\ell,j_\ell}=1$ and for $k<m$,
$\sum_{j_k=1}^{n_k} v_{k, j_k} a_{k,j_k} = \scalarp{\vs_k, \as_k} = 0$,
we have $\scalarp{\M_k^*(\vs_k), \M_h^*(\vs_h) }_{\As}=0$ and hence
the statement follows.
\end{proof}

\thmBGrate*

\begin{proof}
We start by recalling the two formulas
\begin{equation}
\label{eq:20211128a}
\Xs^t = \exp\Big( -\frac{\Cs}{\eta} + \Vs^t \Big) \odot \alpha
\quad \text{and} \quad
\Xs^\star = \exp\Big( -\frac{\Cs}{\eta} + \Vs^\star \Big) \odot \alpha,
\end{equation}
where $\alpha := \bigotimes_{k=1}^m \as_k$, $\Vs^t := \bigoplus_{k=1}^m \vs^t_k$,
and $\Vs^\star := \bigoplus_{k=1}^m \vs_k^\star$. Moreover, since for every $(\lambda_k)_{k \in \N} \in \R^{\N}$ such that $\sum_{k=1}^m \lambda_k = 0$, we have $\bigoplus_{k=1}^m (\vs^t_k+ \lambda_k) = \bigoplus_{k=1}^m \vs^t_k$ and
$\bigoplus_{k=1}^m (\vs_k^\star + \lambda_k) = \bigoplus_{k=1}^m \vs_k^\star$,
we can choose the dual variables $(\vs^t_k)_{k \in [n_k]}$ and 
$(\vs^\star_k)_{k \in [n_k]}$ so that 
\begin{equation}
(\forall\, k=1, \dots, m-1)\quad 
\scalarp{\vs^t_k, \as_k} = 0
\quad \text{and} \quad
\scalarp{\vs^\star_k, \as_k} = 0.
\end{equation}
First, observe that Pythgora's theorem yields  that $\KL_{\Pi}(\xs^{t+1}) = \KL_{\Pi}(\pi^t) - \KL_{\Pi_{(k_t,L_t)}}(\xs^t) \leq \KL_{\Pi}(\pi^t)$, which implies that, for every $t\geq0$,
 $D_{\Ent^*}(\log \pi^t,\log\pi^*) = \KL(\pi^\star,\pi^t)\leq  \KL_{\Pi}(\pi^0) <+\infty$. 
Howevery, since $\Ent^*$ is a Legendre function, 
the sublevel sets of  $D_{\Ent^*}(\cdot, \log\pi^*)$ are bounded,
and hence the sequence $(\log \xs^t)_{t \in \N}$ is bounded in $\Xx$.
Now, since the first of \eqref{eq:20211128a} yields that 
$\log \xs^t = - \Cs/\eta + \Vs^t + \log \alpha$, we have that also the sequence $(\Vs^t)_{t \in \N}$ is bounded in $\Xx$. Thus let $M_1>0$ be such that
\begin{equation*}
\norm{\Vs^\star}_{\infty},  \norm{\Vs^t}_{\infty} \leq M_1
\quad (\forall\, t \in \N).
\end{equation*}
Then, recalling \eqref{eq:20211128a}, 
\begin{equation*}
\frac{\xs^t}{\alpha} = \exp\Big( - \frac{\Cs}{\eta} + \Vs^t \Big)
\geq \exp\Big( - \frac{\norm{\Cs}_\infty}{\eta} - \norm{\Vs^t}_{\infty} \Big)
\geq \exp\Big( - \frac{\norm{\Cs}_\infty}{\eta} -M_1 \Big)
\end{equation*}
and 
\begin{equation*}
\frac{\xs^\star}{\alpha} = \exp\Big( - \frac{\Cs}{\eta} + \Vs^\star \Big)
\geq \exp\Big( - \frac{\norm{\Cs}_\infty}{\eta} - \norm{\Vs^\star}_{\infty} \Big)
\geq \exp\Big( - \frac{\norm{\Cs}_\infty}{\eta} -M_1 \Big)
\end{equation*}
and hence
\begin{equation}
\label{eq:20211128g}
\exp\Big( - \frac{\norm{\Cs}_\infty}{\eta} -M_1 \Big) \leq \min\Big\{ \frac{\xs^t}{\alpha}, \frac{\xs^\star}{\alpha} \Big\}
\quad (\forall\, t \in \N).
\end{equation}
Let $t \in \N$, $k \in [m]$ and $L \subset [n_k]$. It follows from \eqref{eq:KL_proj} that
\begin{equation}
(\forall\, j \in \mathcal{J})\quad
(\Pc_{\Pi_{(k,L)}}(\Xs^t))_{j} = 
\xs^t_{j} \times
\begin{cases}
\dfrac{a_{k,j_k}}{{\M_{k}(\xs)}_{j_k}} &\text{if } j_k \in L,  \\[2ex]
1 & \text{otherwise}
\end{cases}
\end{equation}
and hence
\begin{equation}
\label{eq:20211128b}
(\forall\, j \in \mathcal{J})\quad
\frac{(\Pc_{\Pi_{(k,L)}}(\Xs^t))_j}{\alpha_j}
\leq \frac{\xs^t_j}{\alpha_j}  \max\bigg\{ 1, \dfrac{a_{k,j_k}}{{\M_{k}(\xs)}_{j_k}}  \bigg\}.
\end{equation}
Now, since $\Cs \geq 0$, we have
\begin{equation}
\label{eq:20211128c}
\frac{\Xs^t}{\alpha} = \exp\Big( -\frac{\Cs}{\eta} + \Vs^t \Big) \leq \exp(\Vs^t) \leq 
\exp(\norm{\Vs}_{\infty}) \leq \exp(M_1)
\end{equation}
and
\begin{align}
\label{eq:20211128d}
\nonumber
\dfrac{{\M_{k}(\xs)}_{j_k}}{a_{k,j_k}}
&= \dfrac{\sum_{j_{-k} \in \mathcal{J}_{-k}} \Xs^t_{(j_{-k}, j_k)}}{a_{k,j_k}}\\
\nonumber
&= \dfrac{\sum_{j_{-k} \in \mathcal{J}_{-k}} \exp\big( - \Cs_{(j_{-k}, j_k)}/\eta + \Vs^t_{(j_{-k}, j_k)}\big) \prod_{h=1}^m a_{h,j_h}}{a_{k,j_k}}\\
\nonumber
&= \sum_{j_{-k} \in \mathcal{J}_{-k}} \exp\big( - \Cs_{(j_{-k}, j_k)}/\eta + \Vs^t_{(j_{-k}, j_k)} \big) \prod_{h\neq k}^m a_{h,j_h}\\
\nonumber
& \geq \exp\big( - \norm{\Cs}_{\infty}/\eta - M_1 \big) 
\sum_{j_{-k} \in \mathcal{J}_{-k}} \prod_{h\neq k}^m a_{h,j_h}\\
\nonumber
& = \exp\big( - \norm{\Cs}_{\infty}/\eta - M_1 \big) \prod_{h\neq k}^m 
\bigg(\sum_{j_h = 1}^{n_h} a_{h, j_h} \bigg)\\
& = \exp\big( - \norm{\Cs}_{\infty}/\eta - M_1 \big),
\end{align}
since $\sum_{j_h = 1}^{n_h} a_{h, j_h} = 1$. Therefore, by \eqref{eq:20211128b},
\eqref{eq:20211128c}, and \eqref{eq:20211128d},
\begin{equation*}
\frac{\Pc_{\Pi_{(k,L)}}(\Xs^t)}{\alpha}
\leq \exp (M_1) \exp\big( \norm{\Cs}_{\infty}/\eta + M_1 \big) =
\exp\big( \norm{\Cs}_{\infty}/\eta + 2 M_1 \big)
\end{equation*}
and hence, recalling \eqref{eq:20211128c},
\begin{equation}
\label{eq:20211128i}
\max\bigg\{ \frac{\Xs^t}{\alpha}, \frac{\Pc_{\Pi_{(k,L)}}(\Xs^t)}{\alpha} \bigg\}
\leq \exp\big( \norm{\Cs}_{\infty}/\eta + 2 M_1 \big).
\end{equation}
We now prove that
\begin{equation}
\label{eq:20211128e}
\exp\big(- 2 \norm{\Cs}_{\infty} - 3 M_1 \big) b^{-1}_\tau 
\KL_{\Pi}(\xs^{t}) = \KL_{\Pi_{(k_t, L_t)}}(\xs^{t}) = 
\max_{(k,L) \in \mathcal{I}(\tau)} \KL (\Pc_{\Pi_{(k,L)}}(\Xs^t), \Xs^t ).
\end{equation}
From this inequality it will follow, using 
the Pythgora's theorem  $\KL_{\Pi}(\xs^{t+1}) + \KL_{\Pi_{(k_t,L_t)}}(\xs^t) = \KL_{\Pi}(\pi^t)$, that
\begin{equation}
\exp( - 2 \norm{\Cs}_{\infty} - 3 M_1) b^{-1}_\tau \KL_{\Pi}(\xs^{t})  = \KL_{\Pi}(\pi^t) - 
\KL_{\Pi}(\xs^{t+1})
\end{equation}
and hence
\begin{equation}
\KL_{\Pi}(\xs^{t+1}) \leq \big(1 - \exp( - 2 \norm{\Cs}_{\infty} - 3 M_1) b^{-1}_\tau \big)
\KL_{\Pi}(\xs^{t}),
\end{equation}
which gives the statement. Thus, it remains to prove \eqref{eq:20211128e}.
Let $t \in \N$ and, for the sake of brevity set
\begin{equation}
\label{eq:20211128h}
\pi := \pi^t,\quad\pi_k^h :=  \Pc_{\Pi_{(k,L_k^h)}}(\xs^t)\quad \vs_k := \vs_k^\star - \vs^t_k,\quad \text{ and }\quad  \vs_k^h:=J^{(k)*}_{L^h_k} \vs_k.
\end{equation}
Let, for every $k \in [m]$, $(L_k^h)_{1 \leq h \leq s_k}$ be a partition of $[n_k]$
made of non empty sets of cardinality exactly $\tau_k$ possibly except for the last one,
such that $L_{k_{t-1}}^1 = L_{t-1}$, where $s_k = \lceil n_k/\tau_k \rceil$. 
Then, it follows from \eqref{eq:20211128a} that
\begin{equation}
\frac{\xs^\star}{\xs} = \frac{\exp(-\Cs/\eta)\exp(\Vs^\star)}{\exp(-\Cs/\eta) \exp(\Vs^t)}
= \exp(\Vs^\star - \Vs^t).
\end{equation}
Hence, recalling that $\M^*(\vs_1, \dots, \vs_m) = \sum_{k=1}^m \M_k^* \vs_k = \bigoplus_{k=1}^m \vs_k = \bigoplus_{k=1}^m (\vs_k^\star - \vs_k^t)$, we have
\begin{align}
\nonumber\KL_{\Pi} (\xs) + \KL(\xs, \xs^\star) &= 
\KL(\xs^\star, \xs)+\KL(\xs, \xs^\star) \\
\nonumber&= \scalarp{\xs^\star - \xs, \log(\xs^\star/\xs)}\\
\nonumber&= \scalarp{\xs^\star - \xs, \Vs^\star - \Vs^t}\\
\nonumber&= \scalarp{\xs^\star - \xs, \M^*(\vs_1, \dots, \vs_m)}\\
\label{eq:20211129a}& = \sum_{k=1}^m \scalarp{\xs^\star - \xs, \M^*_k(\vs_k)}.
\end{align}
Moreover, recalling that $\M_{(k,L_k^h)} = J^{(k)*}_{L_k^h} \circ \M_k^*$ and 
$\vs_k^h=J^{(k)*}_{L^h_k} (\vs_k)$, we have
\begin{equation}
\M_k^*(\vs_k) = 
\M_k^*\bigg(\sum_{h=1}^{n_h} J^{(k)}_{L_k^h} \circ J^{(k)*}_{L_k^h}(\vs_k)\bigg) = 
\sum_{h=1}^{n_h} \M^*_{(k,L_k^h)}(\vs_k^h)
\end{equation}
and hence
\begin{equation}
\KL_{\Pi} (\xs) + \KL(\xs, \xs^\star) =
\sum_{k=1}^m \sum_{h=1}^{s_k} \scalarp{\xs^\star - \xs, \M^*_{(k,L_k^h)}(\vs_k^h)}.
\end{equation}
Now, recalling the general definition of $\Pi_{(k,L)}$ in \eqref{eq:transpoly2},
since $\xs_k^h$ and  $\xs^\star$ both belong to  $\Pi_{(k,L_k^h)}$, 
we have that $\xs^\star - \xs_k^h \in \Ker(\M_{(k,L_k^h)}) = \range(\M^*_{(k,L_k^h)})^{\perp}$ and hence
\begin{equation*}
\scalarp{\xs^\star - \xs, \M^*_{(k,L_k^h)}(\vs_k^h)} = 
\scalarp{\xs^\star - \xs_k^h, \M^*_{(k,L_k^h)}(\vs_k^h)} +
\scalarp{\xs_k^h - \xs, \M^*_{(k,L_k^h)}(\vs_k^h)} =
\scalarp{\xs_k^h - \xs, \M^*_{(k,L_k^h)}(\vs_k^h)}
\end{equation*}
and hence
\begin{align}
\label{eq:20211128f}
\nonumber 
\KL_{\Pi} (\xs) + \KL(\xs, \xs^\star) &= 
\sum_{k=1}^m \sum_{h=1}^{s_k} \scalarp{\xs_k^h - \xs, \M^*_{(k,L_k^h)}(\vs_k^h)}\\
&= \sum_{k=1}^m \sum_{h=1}^{s_k} 
\scalarp{\As^{-1} (\xs_k^h - \xs), \M^*_{(k,L_k^h)}(\vs_k^h)}_{\As},
\end{align}
where $\As$ is the positive diagonal operator defined in Lemma~\ref{lem:20211128a}.
Now, it follows from \eqref{eq:20211128g}, Lemma~\ref{fact:KLbound},
and Lemma~\ref{lem:20211128a} that 
\begin{align*}
\KL(\xs^t, \xs^\star) &\geq (1/2)\exp(- \norm{\Cs}_\infty/\eta - M_1) 
\norm{\log \xs^\star - \log \xs^t}^2_{\As}\\
&= (1/2)\exp(- \norm{\Cs}_\infty/\eta - M_1) 
\norm{\Vs^\star - \Vs^t}^2_{\As}\\
&= (1/2)\exp(- \norm{\Cs}_\infty/\eta - M_1) 
\norm{\M^*(\vs_1, \dots, \vs_m)}^2_{\As}\\
&= (1/2)\exp(- \norm{\Cs}_\infty/\eta - M_1) 
\sum_{k=1}^m \norm{\vs_k}^2_{\As_k}.
\end{align*}
Moreover, recalling the definition of $\vs_k^h$ in \eqref{eq:20211128h}, since $\vs_k = 
\sum_{h=1}^{s_k} J_{(k,L_k^h)}\vs^h_k$ and $(J_{(k,L_k^h)}\vs^h_k)_{h \in [s_k]}$
is a finite orthogonal sequence in $\R^{n_k}$ w.r.t.~the metric $\scalarp{\cdot, \cdot}_{\As_k}$, we have
\begin{equation}
\norm{\vs_k}^2_{\As_k} = 
\sum_{h=1}^{s_k} \norm{J_{(k,L_k^h)}\vs^h_k}^2_{\As_k}
= \sum_{h=1}^{s_k} \norm{\M_k^* J_{(k,L_k^h)}\vs^h_k}^2_{\As}
= \sum_{h=1}^{s_k} \norm{\M_{(k, L_k^h)}^* \vs^h_k}^2_{\As},
\end{equation}
where we used equation \eqref{lem:20211128a_eq1} from Lemma~\ref{lem:20211128a} applied to $J_{(k,L_k^h)}\vs^h_k$
and the fact that, by definition, $\M_{(j,L_k^h)} = J^{(k)*}_{L_k^h} \M_k$.
Overall we get that
\begin{equation*}
\KL(\xs^t, \xs^\star) \geq 
(1/2)\exp(- \norm{\Cs}_\infty/\eta - M_1) 
\sum_{k=1}^m \sum_{h=1}^{s_k} \norm{\M_{(k, L_k^h)}^* \vs^h_k}^2_{\As}
\end{equation*}
and hence \eqref{eq:20211128f} yields
\begin{align*}
\KL_{\Pi} (\xs)  &\leq \sum_{k=1}^m \sum_{h=1}^{s_k} 
\scalarp{\As^{-1} (\xs_k^h - \xs), \M^*_{(k,L_k^h)}(\vs_k^h)}_{\As} - \KL(\xs, \xs^\star)\\
&\leq \sum_{k=1}^m \sum_{h=1}^{s_k} 
\scalarp{\As^{-1} (\xs_k^h - \xs), \M^*_{(k,L_k^h)}(\vs_k^h)}_{\As} - 
\frac1 2 \exp(- \norm{\Cs}_\infty/\eta - M_1) \norm{\M_{(k, L_k^h)}^* \vs^h_k}^2_{\As}\\
&\leq \frac{\exp(\norm{\Cs}_\infty/\eta + M_1)}{2} \sum_{k=1}^m \sum_{h=1}^{s_k} \norm{\As^{-1}(\xs_k^h - \xs)}^2_{\As}\\
&= \frac{\exp(\norm{\Cs}_\infty/\eta + M_1)}{2} \sum_{k=1}^m \sum_{h=1}^{s_k} \norm{(\xs_k^h - \xs)}^2_{\As^{-1}},
\end{align*}
where in the last inequality we used the Young-Fenchel inequality
$\scalarp{a,b}_\As \leq \tfrac{\mu}{2} \norm{a}_\As^2
+\tfrac{1}{2 \mu} \norm{b}_\As^2$.
Now, recalling that we set $\Xs = \Xs^t$ and $\Xs_k^h = \Pc_{\Pi_{(k,L_k^h)}}(\Xs^t)$, it follows from \eqref{eq:20211128i} and Lemma~\ref{fact:KLbound} that
$\frac{1}{2}\norm{ \pi_k^h -\xs}_{\Theta^{-1}}^2 \leq  \exp\big( \norm{\Cs}_{\infty}/\eta + 2 M_1 \big)\KL(\pi_k^h,\xs)$, and consequently 
\begin{align*}
\KL_{\Pi}(\xs^t) 
& \leq \exp\big( 2\norm{\Cs}_{\infty}/\eta + 3 M_1 \big)\sum_{k\in[m]}\sum_{h\in[s_k]} \KL(\pi_k^h, \xs) \\
&\leq \exp\big( 2\norm{\Cs}_{\infty}/\eta + 3 M_1 \big) \bigg( \sum_{k\in[m]} s_k - 1 \bigg) \max_{k\in[m]}\max_{ h\in[s_k]}\KL(\pi_k^h, \xs)\\
& = \exp\big( 2\norm{\Cs}_{\infty}/\eta + 3 M_1 \big) 
\bigg(  \sum_{k\in[m]} \lceil n_k / \tau_k \rceil - 1 \bigg) \max_{k\in[m]}\max_{ h\in[s_k]}\KL(\Pc_{\Pi_{(k,L_k^h)}}(\xs^t), \xs^t)\\
& \leq (b_\tau-1) \exp\big( 2\norm{\Cs}_{\infty}/\eta + 3 M_1 \big) \max_{(k,L) \in \Ic(\tau)}\KL(\Pc_{\Pi_{(k,L)}}(\xs^t), \xs^t),
\end{align*}
where in the second inequality we used that, 
for $k=k_{t-1}$ and $h=1$, $\xs_{k}^h = \Pc_{\Pi_{(k_{t-1},L_{t-1})}}(\xs^t) 
= \xs^t$ (since by definition $\xs^t \in \Pi_{(k_{t-1},L_{t-1})}$), so that $\KL(\pi_k^h, \xs)=0$.
This proves \eqref{eq:20211128e} and the proof is complete.
\end{proof}

We now provide a result concerning the convergence of numerical sequences
which is critical to analyze the iteration complexity of the algorithm. This result
has been first showed implicitly in \cite{DGK2018}. We provide here a more explicit version together with a complete proof for the reader's convenience.

\begin{lemma}
\label{lem:complexity}
Let $M,C>0$ and 
let $(\delta_t)_{t \in \N}$ and $(d_t)_{t \in \N}$ be two sequences of positive numbers such that,
for every $t \in \N$,
\begin{enumerate}[label={\rm (\roman*)}]
\item\label{lem:complexity_i} 
$\delta_t - \delta_{t+1} \geq \bigg(\dfrac{d_t}{C}\bigg)^2$,
\item\label{lem:complexity_ii} 
$\delta_t \leq M d_t$.
\end{enumerate}
Let $\varepsilon>0$ and set $\bar{t} = \min \{t \in \N \,\vert\, d_t \leq \varepsilon\}$.
Then $\bar{t} \leq 1 + 2 M C^2/\varepsilon$.
\end{lemma}
\begin{proof}
Items \ref{lem:complexity_i} and \ref{lem:complexity_ii} imply that
\begin{equation*}
\delta_t - \delta_{t+1} \geq \bigg(\frac{\delta_t}{M C} \bigg)^2.
\end{equation*}
Therefore, since $\delta_t \geq \delta_{t+1}$, we have
\begin{equation*}
\delta_t - \delta_{t+1} \geq \frac{\delta^2_t}{M^2 C^2}
\geq \frac{\delta_t \delta_{t+1}}{M^2 C^2}
\end{equation*}
and hence, dividing by $\delta_t\delta_{t+1}$,
\begin{equation*}
\frac{1}{\delta_{t+1}} - \frac{1}{\delta_{t}} \geq \frac{1}{M^2 C^2}.
\end{equation*}
Thus,
\begin{equation*}
\frac{1}{\delta_t} - \frac{1}{\delta_0} = \sum_{i=0}^{t-1} 
\bigg( \frac{1}{\delta_{i}} - \frac{1}{\delta_{i+1}} \bigg) \geq \frac{t}{M^2 C^2}
\end{equation*}
and hence we get the following rate of convergence for 
the sequence $(\delta_t)_{t \in \N}$
\begin{equation}
\label{eq:rate}
\delta_t \leq \bigg( \frac{1}{\delta_0} + \frac{t}{M^2 C^2} \bigg)^{-1}.
\end{equation}
Now, we let $\delta \in \left]0,\delta_0\right]$. We wish to determine the number of iterations such that $\delta_t \leq \delta$. It follows from \eqref{eq:rate} that
\begin{equation}
\bigg( \frac{1}{\delta_0} + \frac{t}{M^2 C^2} \bigg)^{-1} \leq \delta \ \Leftrightarrow\ \
\frac{1}{\delta_0} + \frac{t}{M^2 C^2} \geq \frac{1}{\delta} \ \Leftrightarrow\ \
t \geq M^2 C^2 \bigg( \frac{1}{\delta} - \frac{1}{\delta_0} \bigg).
\end{equation}
This means that if we take $t \geq M^2 C^2 (1/\delta - 1/\delta_0)$, we have $\delta_t \leq \delta$ as desired. So we set $t = \lfloor M^2 C^2 (1/\delta - 1/\delta_0) \rfloor + 1$.
Then, we have $\delta_t \leq \delta$. Now we have to cases. Suppose that $t<\bar{t}$
and let $s \in \N$ be such that $t+s = \bar{t}-1$. Then, for every $i=0,\dots, s$, since 
$t+i < \bar{t}$, we have $d_{t+i}> \varepsilon$ and hence, using \ref{lem:complexity_i},
\begin{equation}
\delta \geq \delta_t - \delta_{\bar{t}} = \sum_{i=0}^s \bigg( \frac{1}{\delta_{t+i}} - \frac{1}{\delta_{t+i+1}}  \bigg) \geq \sum_{i=0}^s \frac{d^2_{t+i}}{M^2 C^2} \geq (s+1)\frac{\varepsilon^2}{M^2 C^2},
\end{equation}
which implies that $s+1 \leq C^2 \delta/\varepsilon^2$. Overall we have
\begin{equation}
\label{eq:20211130a}
\bar{t} = t + s + 1 \leq \bigg\lfloor M^2 C^2 \bigg(\frac{1}{\delta} - \frac{1}{\delta_0}\bigg) \bigg\rfloor + 1 + C^2 \frac{\delta}{\varepsilon^2}
\leq 1+ \frac{M^2 C^2}{\delta} - \frac{M^2 C^2}{\delta_0} + \frac{C^2 \delta}{\varepsilon^2}.
\end{equation}
Note that this inequality is true for any $\delta \in \left]0,\delta_0\right]$. 
Now, suppose that $M \varepsilon \leq \delta_0$. Then we have
\begin{equation}
\bar{t} \leq 1 + \min_{\delta \in \left]0,\delta_0\right]} \bigg(\frac{M^2 C^2}{\delta} + \frac{C^2 \delta}{\varepsilon^2}\bigg) = 1 + 2 \frac{M C^2}{\varepsilon},
\end{equation}
where the minimum is attained at $\delta = M \varepsilon \in \left]0,\delta_0\right]$.
On the other hand, if $\delta_0< M\varepsilon$, then the minimum on the right hand side of \eqref{eq:20211130a} is attained at $\delta= \delta_0$ and hence
\begin{equation}
\bar{t} \leq 1 + \frac{C^2\delta_0}{\varepsilon^2} \leq 1 + \frac{C^2 M \varepsilon}{\varepsilon^2} = 1+ \frac{M C^2}{\varepsilon}.
\end{equation}
In any case, the statement follows.
\end{proof}

\begin{remark}
The statement of Lemma~\ref{lem:complexity} is equivalent to the fact that
the sequence $(\min_{0 \leq s < t} d_s)_{t \in \N}$ converge to zero with rate $\bigO(1/t)$, i.e., that
for every integer $t > 1$,
\begin{equation*}
\min_{0 \leq s < t} d_s \leq \frac{2M C^2}{t-1}.
\end{equation*}
\end{remark}

Next,  we prove the main result on the iteration complexity.

\thmBGiterations*

\begin{proof}
For the sake of brevity let $\bar{b} = \max_{k\in[n]} \lceil n_k / \tau_k \rceil$ and set, for every $t \in \N$, $\delta_t := \KL_{\Pi}(\xs^t)$. Let $t \in \N$ be arbitrary. Recalling \eqref{eq:20211129a}, we have that 
\begin{equation*}
\delta_t = \KL_\Pi(\xs^t) \leq 
\sum_{k\in[m]}  \scalarp{ \xs^\star-\xs^t,\M_k^*(\vs_k^\star - \vs_k^t )} \!=\! \sum_{k\in[m]}  \scalarp{ \as_k-\M_k(\xs^t),\vs_k^\star - \vs_k^t},
\end{equation*} 
which, using Ho\"{o}lder inequality, yields  
\begin{equation}\label{eq:bounds1}
\delta_t  \leq  \sum_{k\in[m]}  \norm{ \as_k-\M_k(\xs^t)}_1\;\norm{\vs_k^\star - \vs_k^t}_\infty \leq M_2 d_t.
\end{equation} 

Now, we prove
\begin{equation}\label{eq:bounds2}
\delta_{t} - \delta_{t+1} \geq  \min\Big\{\frac{d_t^2}{5\bar{b}}, \frac{\delta_t^2}{4 M_2^2 \bar{b}} \Big\}\geq \frac{\delta_t^2}{5 M_2^2\bar{b}}.
\end{equation} 

Let for every $k\in [m]$, $(L_k^h)_{1 \leq h \leq [s_k]}$, $s_k:=\lceil n_{k} / \tau_{k} \rceil$, be a partition of $[n_{k}]$ made of nonempty sets of cardinality exactly $\tau_{k}$, except maybe for the last one, such that $L_{k_t}^1 = L_t$ (not that necessarily the cardinality of $L_t$ is $\tau_{k_t}$). Then, according to the greedy choice of $(k_t,L_t)$ we have that
\begin{equation*}
\bar{b}\, \KL_{\Pi_{(k_t,L_t)}}(\xs^t) \geq \max_{k\in[m]} s_k \max_{h\in[s_k]}\;\KL_{\Pi_{(k,L_k^h)}}(\xs^t) \geq \max_{k\in[m]}\sum_{h\in[s_k]}  \KL_{\Pi_{(k,L_k^h)}}(\xs^t).
\end{equation*} 
Thus, equation \eqref{eq:KL_itstep} of Proposition \ref{prop:KL_proj} yields
\begin{equation}
\label{eq:20211130c}
\bar{b}\, \KL_{\Pi_{(k_t,L_t)}}(\xs^t) \geq \max_{k\in[m]}\sum_{h\in[ s_k ]}  \KL({\as_{k}}_{\lvert L_k^h}, \M_{k}(\xs^t)_{\lvert L_k^h}) = \max_{k\in[m]} \KL(\as_{k}, \M_{k}(\xs^t)).
\end{equation} 
Now Pinsker's inequality guaranties that, for every $k \in [m]$ 
\begin{equation}
\label{eq:20211201d}
\KL(\as_{k}, \M_{k}(\xs^t)) \geq \frac{\norm{\as_{k} - \M_{k}(\xs^t)}_1^2}{\tfrac{2}{3} \norm{\as_{k}}_1 + \tfrac{4}{3} \norm{\M_{k}(\xs^t)}_1} = \frac{\norm{\as_{k} - \M_{k}(\xs^t)}_1^2}{\tfrac{2}{3} + \tfrac{4}{3} \norm{\xs^t}_1}\geq \frac{\norm{\as_{k} - \M_{k}(\xs^t)}_1^2}{2 + \tfrac{4}{3} \norm{\as_k-\M_k(\xs^t)}_1},
\end{equation}
where in the second inequality we used that $\norm{\as_k - \M_k(\xs^t)}_1 
\geq \norm{\M_k(\xs^t)}_1 - \norm{\as_k}_1 =  \norm{\xs^t}_1 - 1$.
Thus, solving the quadratic inequality in $\norm{\as_k - \M_k(\xs^t)}_1\geq0$ we can conclude that 
\begin{equation*}
\norm{\as_k - \M_k(\xs^t)}_1 \leq \tfrac{2}{3} \KL(\as_{k}, \M_{k}(\xs^t)) + \sqrt{ (\tfrac{2}{3} \KL(\as_{k}, \M_{k}(\xs^t)))^2 + 2\KL(\as_{k}, \M_{k}(\xs^t))}.
\end{equation*}

Therefore, if $\max_{k\in[m]}\KL(\as_{k}, \M_{k}(\xs^t)) \leq 1$, then $2+4 d_t / 3 \leq 5$, and consequently, $\max_{k\in[m]}\KL(\as_{k}, \M_{k}(\xs^t)) \geq d_t^2 / 5$, which,
using Pythagoras theorem and \eqref{eq:20211130c}, yields
\begin{equation}
\label{eq:20211201e}
\delta_t - \delta_{t+1} = \KL_{\Pi}(\xs^t) - \KL_{\Pi}(\xs^{t+1})= \KL_{\Pi_{(k_t,L_t)}}(\xs^t) \geq  \frac{d_t^2}{5 \bar{b}}. 
\end{equation}

On the other hand, if $\max_{k\in[m]}\KL(\as_{k}, \M_{k}(\xs^t)) > 1$, it follows again from Pythagoras theorem and \eqref{eq:20211130c}, that $\delta_t - \delta_{t+1} \geq 1/ \bar{b}$. Moreover, since $\delta_t\leq\delta_0\leq M_2 d_0 \leq M_2 (1+\norm{\xs^0}_1)\leq 2 M_2$, we have that $1 \geq \delta_t^2 / (4 M_2^2)$, and \eqref{eq:bounds2} follows.

Now, similarly to what was done in the proof of Lemma~\ref{lem:complexity}
we can derive from \eqref{eq:bounds2} that
\begin{equation}
\delta_t \leq \bigg( \frac{1}{\delta_0} + \frac{t}{5 M_2^2 \bar{b}} \bigg)^{-1}.
\end{equation}
Thus, if we take $r = \lfloor 5 M_2^2 \bar{b} \rfloor + 1$ we have $\delta_r \leq 1$.
Then, by \eqref{eq:KL_itstep} with $L=[n_k]$ and Pythagoras theorem we have
that, for every $t \in \N$,
\begin{equation*}
\mathsf{KL} ({\as_k}, \M_k(\xs^{r+t}))
=\KL( \Pc_{\Pi_{(k,[n_k])}} (\Xs^{r+t}),\Xs^{r+t}) 
\leq \KL( \Xs^\star,\Xs^{r+t}) =\delta_{r+t} \leq \delta_r \leq 1.
\end{equation*}
Thus, $\max_{k \in [m]} \mathsf{KL} ({\as_k}, \M_k(\xs^{r+t})) \leq 1$
and for what we already saw, 
\begin{equation}
\delta_{r+t}-\delta_{r+t+1} \geq \frac{d^2_{r+t}}{5\bar{b}}.
\end{equation}
In the end the sequence $(\delta_{r+t})_{t \in \N}$ satisfies the two 
assumptions of Lemma~\ref{lem:complexity} with $C = \sqrt{5\bar{b}}$ and 
$M=M_2$. Thus, we can conclude that the smallest $t$ so that $d_{r+t} \leq \varepsilon$
satisfies $t \leq 1 + 10 M_2 \bar{b}/\varepsilon$. Hence 
\begin{equation*}
r+t \leq 2 + \frac{10 M_2 \bar{b}}{\varepsilon} + 5 M_2^2 \bar{b}
\leq 2 + \frac{10 M_2 \bar{b}}{\varepsilon} + \frac{5 M_2^2 \bar{b} \eta}{\varepsilon}
= 2 + \frac{5 M_2 \bar{b}}{\varepsilon}( 2 + M_2\eta).
\qedhere
\end{equation*}
\end{proof}

\medskip

The next two results are based on novel bounds on potentials that imply explicit dependence of constant $M>0$ in the global rate \eqref{eq:KL_global} on the given data: $\as_1,\ldots, \as_m$, $\Cs$ and $\eta$.

\thmBGrateBimarginal*

\begin{proof}
Let $\vs^t_k$, $k\in[m]$, $t\geq0$ be given by Algorithm \ref{alg}. Then from Proposition \ref{prop:primaldual} we have that for every $t\geq0$
\begin{equation}\label{eq:pi_t}
    \xs^{t+1}
    = \exp \Big( - \frac{\Cs}{\eta} + \bigoplus_{k=1}^m \vs^{t+1}_k \Big) 
    \odot \bigotimes_{k=1}^m \as_k,
\end{equation} 
with $\vs_k^0 = 0$ and for $t \ge0$, $k\in[m]$
and $j_k\in[n_k]$,
\begin{equation*}
v_{k,j_k}^{t+1} = \begin{cases}
v_{k,j_k}^{t} + \log(a_{k,j_k})-\log( \M_{k}(\xs^t)_{j_k})  & k=k_t,\,j_k\in L_t,\\
v_{k,j_k}^{t} & \text{otherwise}.
\end{cases}
\end{equation*}
So, to bound $\log\Xs^t$, we will bound $\vs^t_k$, $k\in[m]$. Since $\M_{k_t}(\Xs^{t+1})_{j_{k_t}} = a_{k_t,j_{k_t}}$ for all $j_{k_t}\in L_t$, \eqref{eq:pi_t} implies that
\begin{equation*}
1 = \exp( v^{t+1}_{k_t,j_{k_t}} ) \sum_{j_{-k_t} \in \mathcal{J}_{- k_t}} 
\exp\Big( -\Cs_{(j_{-k_t}, j_{k_t})}/\reg + \sum_{k\neq k_t}v^{t}_{k,j_k}\Big)
\prod_{k \neq k_t} a_{k,j_k},
 \end{equation*}
and, hence,
\begin{equation}\label{eq:vt1}
\exp( -v^{t+1}_{k_t,j_{k_t}} ) =  \sum_{j_{-k_t} \in \mathcal{J}_{- k_t}} \!\!\! \exp\Big( -\Cs_{(j_{-k_t}, j_{k_t})}/\reg + \sum_{k\neq k_t}v^{t}_{k,j_k} \Big) \prod_{k\neq k_t} a_{k,j_k}.
\end{equation} 
So, using \eqref{eq:vstar} we obtain that for every $j_{k_t}\in L_t$
\begin{equation*}
\exp( v^{t+1}_{k_t,j_{k_t}} - v^{\star}_{k_t,j_{k_t}} ) 
=  \dfrac{ \displaystyle \sum_{j_{-k_t} \in \mathcal{J}_{- k_t}} \!\!\!  
\exp\Big( -\Cs_{(j_{-k_t}, j_{k_t})}/\reg + \sum_{k\neq k_t}v^{\star}_{k,j_k} \Big) }{ \displaystyle \sum_{j_{-k_t} \in \mathcal{J}_{- k_t}} \!\!  
\exp\Big( -\Cs_{(j_{-k_t}, j_{k_t})}/\reg + \sum_{k\neq k_t} v^{t}_{k,j_k} \Big) }, 
\end{equation*} 
while 
\begin{equation*}
\exp(v^{\star}_{k_t,j_{k_t}} - v^{t+1}_{k_t,j_{k_t}} ) 
= \dfrac{ \displaystyle \sum_{j_{-k_t} \in \mathcal{J}_{- k_t}} \!\!\!\!\!\!\!  
\exp\Big( -\Cs_{(j_{-k_t}, j_{k_t})}/\reg + \sum_{k\neq k_t}v^{t}_{k,j_k} \Big) }{ \displaystyle \sum_{j_{-k_t} \in \mathcal{J}_{- k_t}} \!\!\!\!\!\!\!  
\exp\Big( -\Cs_{(j_{-k_t}, j_{k_t})}/\reg + \sum_{k\neq k_t}v^{\star}_{k,j_k}\Big) }.
\end{equation*} 
However, since in general, $\alpha_i,\beta_i>0$ implies that $ (\sum_i \alpha_i) / (\sum_i\beta_i) \leq \max_i \alpha_i / \beta_i$, the last two equalities give 
\begin{equation*}
\exp( v^{t+1}_{k_t,j_{k_t}} - v^{\star}_{k_t,j_{k_t}} )  
\leq \max_{j_{-k_t} \in \mathcal{J}_{- k_t}} 
\exp\Big(\sum_{k\neq k_t} (v^{\star}_{k,j_k} -   v^{t}_{k,j_k} )\Big) 
\end{equation*}
and
\begin{equation*}
\exp(v^{\star}_{k_t,j_{k_t}} - v^{t+1}_{k_t,j_{k_t}} ) 
\leq \max_{j_{-k_t} \in \mathcal{J}_{- k_t}} 
\exp\Big( \sum_{k\neq k_t} (v^{t}_{k,j_k} - v^{\star}_{k,j_k} ) \Big).
\end{equation*}
Hence, taking the logarithm we obtain that for every $j_{k_t} \in L_t$,
\begin{equation*}
\abs{v^{t+1}_{k_t,j_{k_t}} - v^{\star}_{k_t,j_{k_t}}}  \leq \max_{j_{-k_t} \in \mathcal{J}_{- k_t}} \Big\lvert \sum_{k\neq k_t} (v^{\star}_{k,j_k}  - v^{t}_{k,j_k} ) \Big\rvert 
\leq \max_{j_{-k_t} \in \mathcal{J}_{- k_t}}  
\sum_{k\neq k_t} \abs{v^{\star}_{k,j_k}  - v^{t}_{k,j_k}}
= \sum_{k\neq k_t} \norm{\vs^{t}_{k} - \vs^{\star}_{k}}_{\infty}.
\end{equation*} 
Therefore, since $v^{t+1}_{k,j_k} = v^{t}_{k,j_k}$ if $k\neq k_t$ or $j_{k_t}\notin L_t$,  
\begin{equation*}
\max\Big\{\norm{\vs^{t+1}_{k_t} - \vs^{\star}_{k_t}}_{\infty}, \sum_{k\neq k_t} \norm{\vs^{t+1}_{k} - \vs^{\star}_{k}}_{\infty} \Big\} \leq \max\Big\{\norm{\vs^{t}_{k_t} - \vs^{\star}_{k_t}}_{\infty}, \sum_{k\neq k_t} \norm{\vs^{t}_{k} - \vs^{\star}_{k}}_{\infty} \Big\},
\end{equation*} 
and, since $m=2$,
\begin{equation*}
\max_{k\in[m]} \norm{\vs^{t+1}_{k} - \vs^{\star}_{k}}_{\infty} \leq \max_{k\in[m]} \norm{\vs^{t}_{k} - \vs^{\star}_{k}}_{\infty},
\end{equation*} 
which implies, recalling that $\vs^0 = 0$, 
that, for all $t\geq0$, $\max_{k\in[m]} \norm{\vs^{t}_{k} - \vs^{\star}_{k}}_{\infty} \leq \max_{k\in[m]}\norm{\vs^\star_k}_{\infty}$. 
Now, in view of \eqref{eq:v_star_bound2} in Lemma~\ref{lm:v_star_bound}, we have
\begin{equation}
\label{eq:20211130d}
\max_{k\in[m]}\norm{\vs^\star_k}_{\infty}\leq \frac{3}{2} \frac{\norm{\Cs}_{\infty}}{\reg}
\end{equation}
and hence, since $\norm{\vs^t_k}_{\infty} \leq \norm{\vs^t_k - \vs^\star_k}_{\infty} + \norm{\vs^\star_{k}}_{\infty}$, $\max_{k \in [m]} \norm{\vs_k^t}_{\infty} \leq 2 \max_{k \in [m]} \norm{\vs_k^\star}_{\infty} \leq 3 \norm{\Cs}_{\infty}/\reg$. In the end, since for every $(\vs_k)_{k \in [m]} \in 
\R^{n_1}\times\dots\times\R^{n_m}$ 
\begin{equation}
\label{eq:2021121c}
\bigg\lVert \bigoplus_{k=1}^m \vs_k \bigg\rVert_{\infty}
= \max_{j \in \mathcal{J}} \bigg\lvert \sum_{k=1}^m v_{k,j_k} \bigg\rvert
\leq \max_{j \in \mathcal{J}} \sum_{k=1}^m \abs{v_{k,j_k}} =
\sum_{k=1}^m \norm{\vs_k}_{\infty} \leq m 
\max_{k \in [m]} \norm{\vs_k}_{\infty},
\end{equation}
we can satisfy the boundedness assumptions on the dual variables of Theorem~\ref{thm:rate_global} with $M = 6 \norm{\Cs}_{\infty}/\reg$ and \eqref{eq:KL_global_bi-marginal} follows from \eqref{eq:KL_global}.
Concerning the iteration complexity, again by \eqref{eq:20211130d}, 
since $\max_{k\in[m]} \norm{\vs^{t}_{k} - \vs^{\star}_{k}}_{\infty} \leq \max_{k\in[m]}\norm{\vs^\star_k}_{\infty}$,
we have $\sum_{k\in[m]} \norm{\vs^{t}_{k} - \vs^{\star}_{k}}_{\infty} \leq  3\norm{\Cs}_\infty /\eta$. So, using $\eta\geq \varepsilon$, \eqref{eq:required_iter_mot} follows directly from \eqref{eq:required_iter} with $M_2 = 3\norm{\Cs}_\infty /\eta$.
\end{proof}

\thmBGrateMultimarginal*

\begin{proof}
Using the same notation as in the previous proof, we first show that 
\begin{equation}\label{eq:20210917}
(\forall\,t\in\N)(\forall\,k\in[m])(\forall\, j_k,\ell_k \in [n_{k}])\quad v^{t}_{k,j_k} - v^{t}_{k,\ell_k}  \leq 2\norm{\Cs}_{\infty} / \eta.
\end{equation}

Indeed, since for $t=0$, $ v^{t}_{k,j_k} - v^{t}_{k,\ell_k}  = 0$, we proceed by induction assuming that \eqref{eq:20210917} holds for $t$ and proving it for $t+1$. 
Noting that for every $k\in[m]$, every $j_{-k} \in \mathcal{J}_{-k}$
and every $j_k, \ell_k \in [n_k]$
\begin{equation*}
\Cs_{(j_{-k},j_k)} -  \Cs_{(j_{-k},\ell_k)}  \leq 2 \norm{\Cs}_{\infty},
\end{equation*}
and that $L_t = [n_{k_t}]$, from \eqref{eq:vt1} we have that 
$v^{t+1}_{k_t,j_{k_t}} - v^{t+1}_{k_t,\ell_{k_t}}\leq 2\norm{\Cs}_{\infty} / \eta$
holds for every $j_{k_t},\ell_{k_t}\in[n_{k_t}]$. On the other hand for every $k\neq k _t$ $v^{t+1}_{k} = v^{t}_{k}$, which using the inductive hypothesis \eqref{eq:20210917}
yields $v^{t+1}_{k,j_k} - v^{t+1}_{k,\ell_k}  \leq 2\norm{\Cs}_{\infty} / \eta$.
In any case \eqref{eq:20210917} holds for $t+1$.

Next, let $t\in\N$ and define the normalizing constants 
$\lambda_1^{t+1},\ldots \lambda_m^{t+1}\in\R$ as 
$\lambda_k^{t+1}:= - \scalarp{\as_k, \vs^{t+1}_k}$ 
for $k\neq k_t$, and $\lambda_{k_t}^{t+1}:= - \sum_{k \neq k_t} \lambda_k^{t+1}$. Then denoting $\us_k^{t+1}:=\vs_k^{t+1} + \lambda^{t+1}_k$, $k\in[m]$, since $\sum_{k\in[m]}\lambda_k^{t+1}=0$, we have 
$\bigoplus_{k=1}^m \vs_k^{t+1} = \bigoplus_{k=1}^m \us_k^{t+1}$ and hence,
recalling \eqref{eq:20211108b},
\begin{equation}\label{eq:us}
 \xs^{t+1}
    = \exp \Big( - \Cs/\eta + \bigoplus_{k=1}^m \us^{t+1}_k \Big) 
    \odot \bigotimes_{k=1}^m \as_k,
\end{equation} 
Moreover, from \eqref{eq:20210917} we have that for every $k\neq k_t$ and 
every $j_{k},\ell_{k} \in [n_{k}]$
\begin{equation*}
u^{t+1}_{k,j_{k}} - u^{t+1}_{k,\ell_k} = v^{t+1}_{k,j_k} - v^{t+1}_{k,\ell_k}  \leq 2\norm{\Cs}_{\infty} / \eta,
\end{equation*} 
which, using $\sum_{j\in[n_k]}a_{k,j}=1$ and the fact that 
the $\lambda_k^{t+1}$'s are chosen so that $\scalarp{\as_k, \us_k^{t+1}}=0$
for all $k\neq k_t$, implies
\begin{equation}
\label{eq:20211201a}
- u^{t+1}_{k,\ell_k} = \sum_{j_k\in[n_k]}a_{k,j_k}(u^{t+1}_{k,j_k} - u^{t+1}_{k,\ell_k} )  
\leq 2\norm{\Cs}_{\infty} / \eta ,\quad \ell_k\in[n_k],
\end{equation} 
and 
\begin{equation}
\label{eq:20211201b}
u^{t+1}_{k,j_k} = \sum_{\ell_k\in[n_k]}a_{k,\ell_k}(u^{t+1}_{k,j_k} - u^{t+1}_{k,\ell_k} ) 
\leq 2\norm{\Cs}_{\infty} / \eta,\quad j_k\in[n_k].
\end{equation} 
Therefore, we have obtained that $\norm{\us^{t+1}_k}_\infty \leq 2\norm{\Cs}_{\infty} / \eta $ for $k\neq k_t$. 
On the other hand, similar to what was done in the proof of Theorem~\ref{thm:rate_global_bi_marginal} we can derive that,
for every $j_{k_t} \in L_t = [n_{k_t}]$,
\begin{equation*}
\exp( -u^{t+1}_{k_t,j_{k_t}} ) =  \sum_{j_{-k_t} \in \mathcal{J}_{- k_t}} \!\!\! \exp\Big( -\Cs_{(j_{-k_t}, j_{k_t})}/\reg + \sum_{k\neq k_t}u^{t}_{k,j_k} \Big) \prod_{k\neq k_t} a_{k,j_k}.
\end{equation*} 
Since, recalling \eqref{eq:20211201a} and \eqref{eq:20211201b}, 
\begin{equation*}
\exp(-(2m-1) \norm{C}_{\infty}/\eta) \leq \exp\Big( -\Cs_{(j_{-k_t}, j_{k_t})}/\reg + \sum_{k\neq k_t}u^{t}_{k,j_k} \Big) \leq 
\exp((2m-1) \norm{C}_{\infty}/\eta),
\end{equation*}
and $\sum_{j_{-k_t} \in \mathcal{J}_{- k_t}}\prod_{k\neq k_t} a_{k,j_k} = 1$, we have
\begin{equation*}
\exp(-(2m-1) \norm{C}_{\infty}/\eta) \leq \exp( -u^{t+1}_{k_t,j_{k_t}} ) \leq 
\exp((2m-1) \norm{C}_{\infty}/\eta).
\end{equation*}
Therefore,
\begin{equation}
\exp\big(\abs{u^{t+1}_{k_t,j_{k_t}}}\big) = 
\max\big\{\exp\big( u^{t+1}_{k_t,j_{k_t}} \big), 
\exp\big( -u^{t+1}_{k_t,j_{k_t}} \big) \big\} \leq \exp((2m-1) \norm{C}_{\infty}/\eta)
\end{equation}
and hence 
\begin{equation*}
\norm{\us^{t+1}_{k_t} }_{\infty} \leq (2m-1)\norm{\Cs}_{\infty} / \eta.
\end{equation*} 
Therefore, we have $\sum_{k\in[m]}\norm{\us^t_k}_\infty  \leq (4m-3) \norm{\Cs}_{\infty} / \eta =:M$ and due to \eqref{eq:us} and the computation \eqref{eq:2021121c}, we can use $M_1 = M$ in Theorem~\ref{thm:rate_global} and get \eqref{eq:KL_global_multi-marginal}.
Concerning iteration complexity, recalling \eqref{eq:v_star_bound}
we have that $\sum_{k \in [m]} \norm{\us_k - \vs_k^\star}_{\infty} \leq 2 (4m-3) \norm{\Cs}_{\infty} / \eta$ and hence 
as done in \eqref{eq:bounds1} we have 
\begin{equation*}
\delta_t \leq \frac{2(4m-3) \norm{\Cs}_\infty}{\eta} d_t.
\end{equation*}
Moreover, since $\norm{\xs^t}_1=1$ and $\bar{b}=1$, 
it follows from \eqref{eq:20211130c}, \eqref{eq:20211201d} and \eqref{eq:20211201e}
that
\begin{equation}\label{eq:bounds_multisin}
\delta_{t} - \delta_{t+1} \geq  \frac{d_t^2}{2}.
\end{equation}
Thus, the statement follows from Lemma~\ref{lem:complexity}.
\end{proof}


\begin{thebibliography}{10}

\bibitem{AC2011}  
M. Agueh and G. Carlier, 
\newblock Barycenters in the Wasserstein space,
\newblock{ \em SIAM Journal on Mathematical Analysis}, 43(2) (2011) pp. 904--924.


\bibitem{AB2020}  
J.~Altschuler and E.~Boix-Adsera,
\newblock Hardness results for multimarginal optimal transport problems,
\newblock in {\em arXiv:2012.05398} (2020).

\bibitem{AWR2017}  
J.~Altschuler, J.~Weedm and P.~Rigollet,
\newblock Near-linear time approximation algorithms for optimal transport via Sinkhorn iteration,
\newblock in {\em 
Proceedings of the 31st International Conference on Neural Information Processing Systems} 
(2017) pp. 1961--1971.

\bibitem{BCCNP2015}
J-D.~Benamou, G.~Carlier, M.~Cuturi, L.~Nenna and G.~Peyr\'e,
\newblock Iterative Bregman projections for regularized transportation problems,
\newblock {\em SIAM J. Sci. Comput.}, 37 (2015) pp. A1111--A1138.

\bibitem{BCCN2017}
J-D.~Benamou, G.~Carlier, M.~Cuturi and L.~Nenna, 
\newblock {\em A Numerical Method to solve multimarginal optimal transport problems with Coulomb cost},
\newblock  In: Glowinski R., Osher S., Yin W. (eds) Splitting Methods in Communication, Imaging, Science, and Engineering. Scientific Computation. Springer, (2017).

\bibitem{BRE1967}
L.M.~Bregman,
\newblock The relaxation method of finding the common point of convex sets and its application to the solution of problems in convex programming,
\newblock {\em USSR Comput. Math. \& Math. Phys.}, 7 (1967) pp. 200--217.

\bibitem{Car2021}
G.~Carlier, 
\newblock On the linear convergence of the multimarginal Sinkhorn algorithm, 
\newblock in {\em hal-03176512} (2021).

\bibitem{CMZJST2019}
J.~Cao, L.~Mo, Y.~Zhang, K.~Jia, C.~Shen and M.~Tan,
\newblock Multimarginal Wasserstein GAN, 
\newblock in {\em 33rd Conference on Neural Information Processing Systems}, (2019) pp. 1776--1786.

\bibitem{CL1981}
Y.~Censor and A.~Lent,
\newblock An iterative row-action method for interval convex programming,
\newblock {\em J. Optim. Theory Appl.}, 34, (1981) pp. 321--353.

\bibitem{CR1996}
Y.~Censor and S.~Reich,
\newblock Iteration of paracontractions and firmly nonexpansive operators with applications to feasibility optimization,
\newblock {\em Optimization}, 37, (1996) pp. 323--339.

\bibitem{DPR2018}
A.~Dessein, N~Papadakis and J.-L. Rouas,
\newblock Regularized optimal transport and the ROT mover’s distance,
\newblock{ \em Journal of Machine~Learning~Research} 19, (2018) pp. 1--53.


\bibitem{DG2020}
S.~Di Marino and A.~Gerolin. 
\newblock An optimal transport approach for the Schr\"odinger bridge problem and convergence of Sinkhorn algorithm,
\newblock{ \em J. Sci. Comput.}, 85(2) (2020).


\bibitem{DGK2018}
P.~Dvurechensky, A.~Gasnikov, A.~Kroshnin.
\newblock Computational Optimal Transport: Complexity by accelerated gradient descent is better than by Sinkhorn’s algorithm,
\newblock{ \em Proceedings of the 35th International Conference on Machine Learning}, PMLR 80:1367-1376, 2018.

\bibitem{FP2019}
C.~Frogner, T.~Poggio.
\newblock Fast and Flexible Inference of Joint Distributions from their Marginals
\newblock{ \em Proceedings of the 36th International Conference on Machine Learning}, PMLR 97:2002-2011, 2019.

\bibitem{HZKSC2019}
Z. He, W. Zuo, M. Kan, S. Shan and X. Chen, 
\newblock Attgan: Facial attribute editing by only changing what you want,  
\newblock{ \em IEEE Transactions on Image Processing}, 28(11) (2019) pp. 5464–5478.

\bibitem{KS2021}
V.~Kosti\'c and S.~Salzo,
\newblock The method of Bregman projections in deterministic and stochastic convex feasibility problems,
\newblock {\em arXiv:2101.01704}, (2021).

\bibitem{LHCJ2020}
T.~ Lin, N.~Ho, M.~Cuturi and M.~I.~Jordan,
\newblock On the complexity of approximating multimarginal optimal transport,
\newblock {\em arXiv:1910.00152}, (2020)

\bibitem{LHJ2020}
T.~ Lin, N.~Ho and M.~I.~Jordan,
\newblock On the Efficiency of Sinkhorn and Greenkhorn and Their Acceleration for Optimal Transport,
\newblock {\em arXiv:1906.01437}, (2021)

\bibitem{MB2021}
L. Mi, J. Bento. 
\newblock Multi-marginal optimal transport defines a generalized metric.
\newblock {\em arXiv:2001.11114v4}, 2021.

\bibitem{Pas2015}
B.~Pass, 
\newblock Multimarginal optimal transport: theory and applications,
\newblock {\em  ESAIM: Mathematical Modelling and Numerical Analysis}, 49(6) (2015) pp. 1771--1790.

\bibitem{PC2019}
G.~Peyr\'e and M. Cuturi,
\newblock {\em Computational Optimal Transport: With Applications to Data Sciences},
\newblock Foundations and Trends in Machine Learning, 11(5-6) (2019) pp. 355--607.

\bibitem{STD2019}
T.~Sun and Q.~Tran-Dinh.
\newblock Generalized Self-Concordant Functions: A Recipe for Newton-Type Methods.
\newblock {\em Mathematical Programing}, 178, 145--213, 2019.

\bibitem{TDGU2020}
N. Tupitsa, P. Dvurechensky, A. Gasnikov, and C. Uribe,
\newblock Multimarginal optimal transport by accelerated gradient descent,
\newblock in {\em Proceedings of the 59th IEEE Conference on Decision and Control}, (2020) pp. 6132--6137. 

\bibitem{Vil2009}
C.~Villani, 
\newblock {\em Optimal Transport -- Old and New}
\newblock Springer-Verlag Berlin Heidelberg, (2009).


\end{thebibliography}
\end{document}